\documentclass[12pt]{colt2018} 

\usepackage{mathextra_colt}
\usepackage{subcaption}
\usepackage{graphicx}

\title[Information Directed Sampling and Bandits with Heteroscedastic Noise]{Information Directed Sampling and Bandits\\ with Heteroscedastic Noise}
\usepackage{times}

 \coltauthor{\Name{Johannes Kirschner} \Email{jkirschner@inf.ethz.ch}\AND
 \Name{Andreas Krause} \Email{krausea@inf.ethz.ch}\\
 \addr Department of Computer Science\\ETH Zurich }

\renewcommand{\jmlrBlackBox}{$\square$}

\jmlrvolume{}
\jmlryear{}
\jmlrpages{}
\jmlrproceedings{}{}

\begin{document}
\newtheorem*{lemma*}{Lemma}
\maketitle


\begin{abstract}In the stochastic bandit problem, the goal is to maximize an unknown function via a sequence of noisy evaluations. Typically, the observation noise is assumed to be independent of the evaluation point and to satisfy a tail bound uniformly on the domain; a restrictive assumption for many applications. In this work, we consider bandits with heteroscedastic noise, where we explicitly allow the noise distribution to depend on the evaluation point. We show that this leads to new trade-offs for information and regret, which are not taken into account by existing approaches like upper confidence bound algorithms (UCB) or Thompson Sampling. To address these shortcomings, we introduce a frequentist regret analysis framework, that is similar to the Bayesian framework of \cite{RussoLearningOptimizeInformationDirected2014}, and we prove a new high-probability regret bound for general, possibly randomized policies, which depends on a quantity we refer to as \emph{regret-information ratio}. From this bound, we define a frequentist version of Information Directed Sampling (IDS) to minimize the regret-information ratio over all possible action sampling distributions. This further relies on concentration inequalities for online least squares regression in separable Hilbert spaces, which we generalize to the case of heteroscedastic noise. We then formulate several variants of IDS for linear and reproducing kernel Hilbert space response functions, yielding novel algorithms for Bayesian optimization. We also prove frequentist regret bounds, which in the homoscedastic case recover known bounds for UCB, but can be much better when the noise is heteroscedastic. Empirically, we demonstrate in a linear setting with heteroscedastic noise, that some of our methods can outperform UCB and Thompson Sampling, while staying competitive when the noise is homoscedastic. \end{abstract}

\newcommand{\UCB}{\text{UCB}}
\newcommand{\TS}{\text{TS}}
\newcommand{\F}{\text{F}}
\newcommand{\IDS}{\text{IDS}}
\newcommand{\DIDS}{\text{DIDS}}

\section{Introduction}
In the stochastic bandit problem one seeks to maximize an unknown function through a sequence of noisy evaluations. At each iteration, an algorithm chooses a point of the domain, trying to maximize the sum of rewards obtained. This objective creates a \emph{dilemma of exploration and exploitation}, as one needs to balance acquiring more information about the unknown function and choosing points which lead to a high reward. To make the problem tractable in cases where the domain is infinite or continuous, the class of possible reward functions has to be restricted, for example to linear functions or a reproducing kernel Hilbert space (RKHS). Further, an assumption commonly used to get concentration results, is that the observation noise is (conditionally) independent and satisfies a tail bound. In most of the existing literature, the noise bound is assumed to be known and independent of the evaluation point. For many applications, this is arguably restrictive in several respects, as for instance the noise distribution might indeed depend on the evaluation point, or the noise in the relevant part of the domain is much lower than the uniform upper bound. To name just a few examples, consider reinforcement learning applications, where data is generated via roll-outs, and noise inherently dependents on the exploration policy and the part of the environment explored. We also point to hyper-parameter tuning on physical systems, where the measurement noise often largely varies with the chosen parameter settings. 

In this work, we introduce \emph{stochastic bandits with heteroscedastic noise}, where we explicitly allow the noise distribution to depend on the evaluation point. With heteroscedastic noise, \emph{two key challenges} arise, which are not present in the homoscedastic case. First, to \emph{estimate} the unknown function, we would like to use the noise dependency to construct a more accurate estimator. Intuitively, one should assign more weight to observations which are less noisy, and indeed, for least squares estimation this leads to a variant known as \emph{weighted least squares}. Second, regarding \emph{exploration and exploitation}, an algorithm should in general prefer actions with less noise, as the observations obtained then result in a better estimate of the unknown function. This calls for methods which efficiently trade off information and regret, something of importance beyond the setting we discuss here. By introducing a novel frequentist regret analysis framework, we derive several algorithms which are able to balance the expected regret against the \emph{informativeness} of the action taken at each time step.

\subsection{Related Work}
For a general introduction to the multi-armed bandit problem we refer the reader to the survey of \cite{BubeckRegretAnalysisStochastic2012} and references therein, and we only point out a few key references below. The well-known upper confidence bound strategy (UCB) was first introduced for the multi-armed bandit problem by \cite{AuerFinitetimeanalysismultiarmed2002}, and later extended to linear bandits \cite[]{DaniStochasticlinearoptimization2008, Abbasi-YadkoriImprovedAlgorithmsLinear2011} and Gaussian processes \cite[]{SrinivasGaussianProcessOptimization2010, Abbasi-YadkoriOnlinelearninglinearly2012, ChowdhuryKernelizedMultiarmedBandits2017}. 
Another popular algorithm is Thompson Sampling \cite[]{ThompsonLikelihoodthatOne1933}, where frequentist regret guarantees are known in the linear case \cite[]{AgrawalThompsonSamplingContextual2013, AbeilleLinearThompsonSampling2017} and for RKHS functions \citep{ChowdhuryKernelizedMultiarmedBandits2017}.

A Bayesian analysis of Thompson Sampling was given by \cite{RussoinformationtheoreticanalysisThompson2016}, and this work also motivated \emph{Information Directed Sampling} \cite[IDS,][]{RussoLearningOptimizeInformationDirected2014}. Bayesian IDS has recently been shown to achieve state-of-the-art performance in the sequential matrix completion problem \cite[]{MarsdenSequentialMatrixCompletion2017}.

For the case of finite, linear bandits, \cite{LattimoreEndOptimismAsymptotic2017} provide an asymptotically instance-optimal algorithm, and showed that both UCB and Thompson Sampling can be arbitrarily far away from the lower bound. Our work adds to this line of research insofar as we identify another setting, where standard approaches like UCB fail to appropriately trade off information and regret. 

We are aware of only a few publications addressing heteroscedastic noise in the bandit setting. Most notably is the work of \cite{CowanNormalBanditsUnknown2015} for the case of finite, independent arms with Gaussian noise and unknown variances, where they show both finite time bounds and an asymptotically instance optimal algorithm. Some connections to the resource allocation problem have been pointed out by \cite{LattimoreLinearMultiResourceAllocation2015}, and heavy-tailed noise distributions are discussed by \cite{LattimoreScaleFreeAlgorithm2017}. While Bayesian IDS handles the complications introduced in our setting in principle, there are no frequentist regret guarantees known for this algorithm. In different context, the heteroscedastic noise problem has been studied before; for example in linear least squares regression \cite[]{AitkenIVleastsquares1936},  Gaussian process regression \cite[]{GoldbergRegressionInputdependentNoise1998, KerstingMostLikelyHeteroscedastic2007}, in active learning \cite[]{AntosActivelearningheteroscedastic2010, ChaudhuriActiveHeteroscedasticRegression2017} and Bayesian optimization \citep{AssaelHeteroscedastictreedbayesian2014}.


\subsection{Contributions}
 In this work, we formalize the stochastic bandit problem with heteroscedastic noise and address the challenges we have outlined in the introduction.
\begin{itemize}
	\item We show in Example \ref{exa: ucb and ts fail}, that the widely used UCB or Thompson Sampling strategies do not appropriately account for heteroscedasticity, and can fail to resolve even simple examples. 
	\item In Section \ref{section: general regret framework}, we show a new high-probability regret bound for randomized policies, that depends on the choice of an \emph{information gain function}, and a quantity we call \emph{regret-information ratio}. One can view this derivation as a frequentist analog of the Bayesian framework of \cite{RussoLearningOptimizeInformationDirected2014}; or as a generalization of the standard analysis of UCB and its variants \cite[]{Abbasi-YadkoriImprovedAlgorithmsLinear2011, SrinivasGaussianProcessOptimization2010}. We also introduce a new concentration inequality for supermartingales (Theorem \ref{thm: any time fan concentration}), which might be of independent interest. 
	\item Our framework motivates a frequentist version of \emph{Information Directed Sampling} (IDS), introduced in Section \ref{section: information directed sampling}. This strategy is designed to have a small \emph{regret-information ratio}, as it appears in our main result. Opposed to the Bayesian setting, where the information ratio can be directly calculated using the prior distribution, we instead define IDS to minimize a frequentist surrogate of the regret-information ratio over all sampling distributions. Moreover, we prove several properties of the regret-information ratio, which were previously only known for finite action sets. We also propose a deterministic variant of IDS, which is computationally cheaper while retaining some of the advantages of the randomized version.
	\item We extend known concentration bounds for online least squares regression \citep{Abbasi-YadkoriOnlinelearninglinearly2012} to heteroscedastic noise (Section \ref{section: least squares regregression with heteroscedastic noise}), which allows us to derive several variants of IDS in Section \ref{section: regret guarantees for ids}. Further, we prove regret bounds for IDS, which in the homoscedastic case are comparable to known bounds for UCB, but can be much better when the noise is heteroscedastic.
\item In Section \ref{section: experiments}, we empirically demonstrate that some of our approaches can outperform UCB and Thompson Sampling in a linear setting with heteroscedastic noise.
\end{itemize}
Finally, we emphasize that, even though in this work we primarily focus on bandits with heteroscedastic noise, \emph{Information Directed Sampling} and the surrounding framework are formulated more generally, and could be used to derive new algorithms for different settings as well.

\section{Problem Statement}\label{section: problem statement}

Let $f: \xX \rightarrow \RR$ be a continuous function, that maps from a compact, metric space $\xX$ to the real numbers. Using the terminology of the bandit literature, we refer to $\xX$ as the action set, and call an element $x \in \xX$ an action (or arm). Denote by $\pP(\xX)$ the set of Borel probability measures on $\xX$. Define a policy $\pi = (\pi_t)_{t\geq 1}$ to be a sequence of mappings $\pi_t :  (\xX \times \RR)^{t-1} \rightarrow \pP(\xX)$ from histories to distributions on $\xX$. We say a policy is deterministic, if $\pi_t$ is always supported on a single action. Further, let $(\epsilon_t)_{t\geq 1}$ be a real-valued noise process. Then we define a stochastic sequence $(x_t, y_t)_{t\geq 1}$ with natural filtration $\fF_t = \sigma(x_1, y_1, \dots, x_t, y_t)$, such that $x_t \sim \pi_t(x_1,y_1, \dots, x_{t-1},y_{t-1})$ is the action chosen by the policy and $y_t = f(x_t) + \epsilon_t$ is a noisy evaluation of $f(x_t)$. The goal is to find a policy, which maximizes the cumulative reward $\sum_{t=1}^T f(x_t)$ or, equivalently, minimizes the regret $R_T = \max_{x\in \xX}\sum_{t=1}^T \big( f(x) - f(x_t) \big)$. By introducing the sub-optimality gaps \linebreak $\Delta(x) = \max_{x' \in \xX}f(x') - f(x)$, we can write $R_T = \sum_{t=1}^T \Delta(x_t)$, and we define $S = \max_{x \in \xX} \Delta(x)$.

In order to derive concentration inequalities for least squares estimators, we make use of a tail bound on the noise distribution. In contrast to existing work, we allow this bound to depend on the evaluation point $x \in \xX$. Formally, let $\rho : \xX \rightarrow \RRp$ be a continuous, positive function, such that $\epsilon_t$ is conditionally $\rho(x_t)$-subgaussian, that is for all $t \geq 1$ and $\rho_t = \rho(x_t)$,
\begin{align}
\forall\; \lambda \in \RR, \quad \EE[e^{\lambda\epsilon_t} | \fF_{t-1}, x_t] \leq \exp\left(\frac{\lambda^2 \rho_t^2}{2}\right) \text{ .}\label{eq: noise assumption}
\end{align}
Note that this condition implies that the noise has zero mean, and common examples include Gaussian, Rademacher, and uniform random variables. At this point, we leave open the kind of knowledge the policy has about the noise function $\rho$; one can think of different settings, where either $\rho$ is given, $\rho(x_t)$ is observed at time $t$, or is estimated from the observations.

In what follows, we also rely on a frequentist notion of \emph{information gain}. We assume that we are given a sequence $(I_t)_{t\geq 1}$ of non-negative, continuous functions  $I_t : (\xX \times \RR)^{t-1} \times \xX \rightarrow \RRnn$, and we think of $I_t(x) = I_t(x_1, y_1, \dots, x_{t-1}, y_{t-1}, x)$ as the information gain of evaluating $x \in \xX$ at time $t$. For a given policy $\pi$, we define the \emph{maximum information gain at time $T$},
\begin{align}
\gamma_T = \esssup_{\fF_T}\sum_{t=1}^T I_t(x_t) \text{ ,} 
\end{align}
where the \emph{essential supremum} is taken over the stochastic process $(x_t,y_t)_{t=1}^T$, up to sets of measure zero. A specific example, that we explore in greater detail later, is $I_t(x) = \log(1 + {\sigma_t(x)^2}/{\rho(x)^2})$, which in the Bayesian setting with posterior variance $\sigma_t(x)^2$ corresponds to the mutual information $\II(x;f|\fF_{t-1})$. In this case our definition of $\gamma_T$ relates to the notion from \cite{SrinivasGaussianProcessOptimization2010}.

\begin{example}[Choice between low and high noise]\label{exa: ucb and ts fail} To start with, we take any bandit instance with homoscedastic noise and duplicate the set of actions. Then, we increase the amount of observation noise on each copy, such that for each action, one has the choice between a low and a high noise version. As we will see shortly, standard algorithms fail to resolve this seemingly simple example.
		
Formally, let $\sS $ be a compact set, and define $\xX = \sS \times \{\rho_1, \rho_2\} $ for $0 < \rho_1 < \rho_2$. We denote $(s,\rho') \in \xX$ for $s \in \sS$ and $\rho' \in \{\rho_1, \rho_2 \}$. Further, let  $f = f_{\theta^*} : \xX \rightarrow \RR$ be a parametrized function, such that $f_{\theta^*}(s,\rho_1) = f_{\theta^*}(s, \rho_2)$ for all $s \in \sS$. We use the last coordinate to specify the noise bound $\rho(s,\rho') = \rho'$.
	Naively, any regret guarantee using a uniform noise bound degenerates as $\rho_2 \rightarrow \infty$. This seems unnecessary, as choosing $(s,\rho_1)$ over $(s, \rho_2)$ clearly yields a more efficient solution. Unfortunately, this is not the default choice. For standard algorithms, consider the widely used UCB strategy. Here, a confidence set $C_t$ is constructed (possibly depending on $\rho$), containing the true parameter $\theta^*$ with high probability. UCB then evaluates $x_t^\UCB \in \argmax_{x \in \xX} \max_{\theta \in C_t} f_\theta(x)$. Note that the acquisition function is independent of $\rho(x_t^\UCB)$, therefore UCB has no preference of $\rho_1$ over $\rho_2$. The same is true for Thompson Sampling, where $x_t^\TS \in \argmax_{x \in \xX} f_{\tilde{\theta}}(x)$ is determined using a sample $\tilde{\theta}$ from the posterior distribution of $\theta$.
\end{example}
Beyond this simple example, one can easily come up with more complicated noise structures, which can be exploited by a learning algorithm. As another extreme case, take a linear bandit problem in $\RR^d$, and add a set of $d$ basis vectors to the action set, such that observations from these actions have no (or infinitesimal small) observation noise. By scaling the new basis vectors, one can also achieve that none of the new actions is optimal. Still, if these actions are played first, one can have vanishing regret after $d$ steps, but again UCB and Thompson Sampling fail to take these actions; compare also Example 2 of \cite{RussoLearningOptimizeInformationDirected2014} and Example 4 of \cite{LattimoreEndOptimismAsymptotic2017}. 

\section{A Frequentist Regret Framework}\label{section: general regret framework}
In this section we introduce a novel high-probability regret bound for general, possibly randomized policies.  The only assumption we will use is that $S = \max_{x \in \xX} \Delta(x)$ is a bound on the sub-optimality gaps, and we emphasize that the results of this section are independent of any assumptions on the noise process. The framework can be seen as a frequentist analog of the Bayesian analysis of \cite{RussoLearningOptimizeInformationDirected2014}. Similarly, as their bound is stated in terms of the entropy of the optimal action, our bounds depend on the total information gain $\gamma_T$, defined through the choice of information gain functions $I_t$. Alternatively, our framework can be understood as a generalization of the standard regret analysis of UCB and its variants \cite[]{Abbasi-YadkoriImprovedAlgorithmsLinear2011, SrinivasGaussianProcessOptimization2010}, and in fact, our framework recovers the usual guarantees for UCB. Similar as in the analysis for UCB, we apply Cauchy-Schwarz's inequality to separate the expected {information gain} from the expected instantaneous regret. Analogously to \cite{RussoLearningOptimizeInformationDirected2014}, this motivates the definition of the \emph{regret-information ratio} at time $t$, 
\begin{align}
\Psi_t: \pP(\xX) \rightarrow \RRnn, \quad \mu \mapsto \frac{\EE_\mu[\Delta(x)|\fF_{t-1}]^2}{\EE_\mu[I_t(x)|\fF_{t-1}]} \text{ .}
\end{align}	
For a fixed policy $\pi$ with sampling distribution $\pi_t$, we write $\Psi_t = \Psi_t(\pi_t)$, and with slight abuse of notation, we define $\Psi_t(x) = \Psi_t(\delta_x)$ as the regret-information ratio of $x \in \xX$, where $\delta_x$ denotes the Dirac delta distribution at $x$. In the next theorem, we introduce a regret bound for general, randomized policies, which depends on the regret-information ratio $\Psi_t$ and the total information gain $\gamma_T$.
\begin{theorem}[Regret Bound for Randomized Policies]\label{thm: general regret bound}
Let  $S = \max_{x \in \xX} \Delta(x)$, and let $\Psi_t, \gamma_T$ as defined above. Then, for any policy, with probability at least $1 - \delta$, at any time $T \geq 1$,  it holds that
	\begin{align*}
		 R_T &\leq  \tfrac{5}{4}  \sqrt{\sum_{t=1}^T \Psi_t \; \left(2\gamma_T + 4 \gamma_T \log \frac{2}{\delta} + 8 \gamma_T \log(4\gamma_T) + 1\right)} + 4 S \log\left(\tfrac{8\pi^2 T^2}{3\delta} (\log(T) + 1)\right) \text{ .}
\end{align*}
	If also $I_t(x_t) \leq 1$ holds for all $t \geq 1$, then, with probability at least $1-\delta$, at any time $T \geq 1$,
	\begin{align*}
	R_T &\leq\tfrac{5}{4}  \sqrt{\sum_{t=1}^T \Psi_t \; \left(2\gamma_T + 4 \log \frac{2}{\delta} + 8 \log(4) + 1\right)} + 4 S \log\left(\tfrac{8\pi^2 T^2}{3\delta} (\log(T) + 1)\right) \text{ .}
	\end{align*}
\end{theorem}
We remark that the assumption $I_t \leq 1$ can be viewed as a continuity assumption on the information gain, and is not very restrictive. In fact, for the information gain functions we use later, $\gamma_T$ is sublinear in $T$, which in turn implies $I_t \rightarrow 0$, and $I_t \leq 1$ is trivially satisfied after some finite time.

The main difficulty in the proof of the general theorem is to deal with randomized policies. Indeed, for deterministic policies, we have the following, simpler version of the previous theorem.
\begin{theorem}[Regret Bound for Deterministic Policies]\label{thm: general regret bound deterministic}
	Let the sequence $(x_t)_{t=1}^T$ be generated by any deterministic policy. Then, at any time $T \geq 1$, the regret is bounded by $R_T \leq  \sqrt{\sum_{t=1}^T \Psi_t \, \gamma_T}$.

\end{theorem}
\begin{proof}
	Note that for deterministic policies, $x_t$ is predictable with respect to the filtration $\fF_t$, and therefore $\Psi_t = \frac{\Delta(x_t)^2}{I_t(x_t)}$. Hence we can write $R_T = \sum_{t=1}^T \Delta(x_t) = \sum_{t=1}^T \sqrt{\Psi_t I_t(x_t)}$ and apply Cauchy-Schwarz to find $R_T \leq  \sqrt{\sum_{t=1}^T \Psi_t \cdot \sum_{t=1}^T I_t(x_t)} \leq  \sqrt{\sum_{t=1}^T \Psi_t \gamma_T}$.
\end{proof}
We outline the proof of the general theorem, but have to defer details to Appendix \ref{app: proof of main theorem}. For brevity, we denote $\Delta_t = \Delta(x_t)$ and  $I_t = I_t(x_t)$. The central step in the proof, that allows us to apply the definition of $\Psi_t$ and Cauchy-Schwarz as in the deterministic case, is
\begin{align*}
\sum_{t=1}^T \EE[\Delta_t|\fF_{t-1}] = \sum_{t=1}^T \sqrt{\Psi_t} \sqrt{\EE[I_t |\fF_{t-1}]} \leq \sqrt{\sum_{t=1}^T \Psi_t \cdot \sum_{t=1}^T \EE[I_t|\fF_{t-1}]} \text{ .}
\end{align*}
To get there, we have to bound the sum of martingale differences $\Delta_t - \EE[\Delta_t|\fF_{t-1}]$, such that we can express the regret in terms of $\EE[\Delta_t|\fF_{t-1}]$. We attain this at the expense of an additive factor of order $\oO\left(S \log \frac{T}{\delta} \right)$ using a `peeling' argument on Freedman's inequality and the union bound on $T$.

The more challenging part of the proof is to bound the sum of expected information gains $\sum_{t=1}^T \EE[I_t|\fF_{t-1}]$ by $\gamma_T$. Again, we need to bound a martingale difference sequence $\EE[I_t|\fF_{t-1}] - I_t$,  such that we can make use of $\sum_{t=1}^T I_t \leq \gamma_T$. It is possible to apply Freedman's inequality as in the first step, or even an Azuma-Hoeffding type inequality like for example Corollary 2.7 of \cite{FanExponentialinequalitiesmartingales2015}. However, both ways would introduce a multiplicative $\sqrt{\log(T)}$ factor in our leading term through the union bound over $T$. 

To avoid this, note that $g_t = \EE[I_t|\fF_{t-1}]$ is a predictable, non-negative process, which upper bounds the martingale differences $g_t - I_t \leq g_t$. Using the method of mixtures, similar as done by \cite{Abbasi-YadkoriImprovedAlgorithmsLinear2011}, and Corollary 2.7 in \citep{FanExponentialinequalitiesmartingales2015}, we prove a new anytime concentration inequality for bounded supermartingales (Theorem \ref{thm: any time fan concentration} in Appendix \ref{app: concentration inequalites for martingales}). As a consequence, we get the following lemma, showing that, for any non-negative stochastic process $X_t$, with high-probability the sum of conditional means $\sum_{t=1}^T \EE[X_t|\fF_{t-1}]$ is not much larger than $\sum_{t=1}^T X_t$.
\begin{lemma}[Concentration of conditional mean]\label{lemma: concentration of conditinal means} Let $X_t$ be any non-negative stochastic process adapted to a filtration $\{\fF_t\}$, and define $m_t = \EE[X_t|\fF_{t-1}]$ and $M_T = \sum_{t=1}^T m_t$. Further assume that $X_t \leq b_t$ for a fixed, non-decreasing sequence $(b_t)_{t \geq 1}$ and let $(l_t)_{t\geq 1}$ be any fixed, positive sequence. Then, with probability at least $1-\delta$, for any $T \geq 1$,
	\begin{align*}
	\sum_{t=1}^T m_t - X_t &\leq \sqrt{2(b_T M_T + l_T) \log\left(\frac{1}{\delta} \frac{(b_T M_T + l_T)^{1/2}}{l_T^{1/2}}\right)} \text{ .}
	\end{align*}
	Further, if $b_T \geq 1$, with probability at least $1-\delta$ for any $T \geq 1$ it holds that,
	\begin{align*}
	\sum_{t=1}^T m_t \leq 2\sum_{t=1}^T X_t + 4 b_T \log\frac{1}{\delta} + 8  b_T\log(4 b_T) + 1
	\end{align*}
\end{lemma}
To prove the first inequality, we use the Bhatia-Davis inequality (Lemma \ref{lemma: bhatia-davis}, Appendix \ref{app: useful inequalities}) to upper bound the sum of conditional variances $\sum_{t=1}^T \Var[X_t|\fF_{t-1}] \leq \sum_{t=1}^T m_t(b_t-m_t) \leq b_T \sum_{t=1}^T m_t$, and for the second part of the lemma, note that the right-hand side of the first inequality depends only sublinearly on $M_T$. All details on the concentration results can be found in Appendix \ref{app: concentration inequalites for martingales}.

Applied to the sequence $I_t$, the last equation then gives the main result for the case $I_t \leq 1$, and the unrestricted case follows by using a more conservative bound $I_t \leq \gamma_t$ instead. We also remark that the $\log\frac{1}{\delta}$ factor introduced in the last step cannot be completely avoided, as one would expect for general randomized policies (see also Example \ref{exa: counter example conditional means almost surely} in Appendix \ref{app: proof of main theorem}).

\pagebreak
\section{Frequentist Information Directed Sampling}\label{section: information directed sampling}
As the regret bound given in the previous section primarily depends on $\sum_{t=1}^T \Psi_t$, we look for policies such that the regret-information ratio is as small as possible. Unfortunately, $\Psi_t$ cannot be directly controlled, as it depends on $\Delta(x_t)$ and therefore on the unknown value $\max_{x \in \xX} f(x)$. However, if a confidence band $[l_t(x), u_t(x)]$ is available, containing the true function values $f(x)$ with probability $1-\delta$, one can construct an upper bound $\Delta_t^+(x) =  \max_{x'} u_t(x') - l_t(x)$, such that $\Delta(x) \leq \Delta_t^+$ also holds with probability $1-\delta$. Such confidence bounds are known for different estimators and function classes, and we will discuss some variants in the next section. For now we rely on $\Delta_t^+$ as being given to define a \emph{surrogate of the regret-information ratio},
\begin{align*}
\Psi_t^+(\mu) : \pP(\xX) \rightarrow \RR, \quad \mu \mapsto \frac{\EE_\mu[\Delta_t^+(x) |\fF_{t-1}]^2}{\EE_\mu[I_t(x) | \fF_{t-1}]} \text{ .}
\end{align*}
Clearly, by our assumption on $\Delta_t^+$, $\Psi_t(\mu) \leq \Psi_t^+(\mu)$ holds with probability $1-\delta$ for any $\mu \in \pP(\xX)$. We define \emph{Information Directed Sampling} (IDS) to be a policy $\pi_{\text{IDS}}$, which depends on the choice of information functions $(I_t)_{t\geq1}$, such that at any time $t$,
\begin{align}
\pi_t^\IDS \in \argmin_{\mu \in \pP(\xX)} \Psi_t^+(\mu) \text{ .}
\end{align}
Using Prokhorov's Theorem \cite[]{ProkhorovConvergenceRandomProcesses1956}, we show in Lemma \ref{lemma: minimizer of psi exists} in Appendix \ref{app: properties of the regret-information ratio}, that a minimizer $\mu \in \pP(\xX)$ of $\Psi_t^+(\mu)$ always exists for compact $\xX$, assuring that IDS is well-defined. Like the Bayesian counterpart in \citep{RussoLearningOptimizeInformationDirected2014}, the regret-information ratio $\Psi_t(\mu)$ and its surrogate  $\Psi_t^+(\mu)$ are convex functions of $\mu$, which is a direct consequence of the fact that $h(x,y) = \frac{x^2}{y}$ is convex on $\RR\times \RRp$ (see Lemma \ref{lemma: psi is convex}). Moreover, we have the following result, which might also have far-reaching algorithmic implications.
\begin{lemma}[Compare Proposition 6 in \citep{RussoLearningOptimizeInformationDirected2014}]\label{lemma: minimizer supported on at most two atoms}
 There exists a minimizing distribution $\mu^* \in \argmin_{\mu \in \pP(\xX)} \Psi_t^+(\mu)$, which is supported on at most two actions.
\end{lemma} 
The lemma is proven for finite action sets in \citep{RussoLearningOptimizeInformationDirected2014}. We address the technical challenges of having a compact set $\xX$ in Appendix \ref{app: properties of the regret-information ratio}, and show that even for continuous action sets, a simple parameterization of the sampling distribution suffices to find a minimizer of $\Psi_t^+$. In settings with finitely many actions, this allows to minimize $\Psi_t^+(\mu)$ by iterating over all pairs of actions, or one can use convexity of $\Psi_t^+$ to employ gradient descent on the probability simplex. 

By the previous lemma, IDS can be chosen in such a way, that the sampling distribution $\pi_t^\IDS$ is supported on at most two actions at any time. That this is necessary in general can easily be seen in a simple example. However, one can still directly minimize the regret-information ratio over distributions supported on a single action. We define \emph{Deterministic Information Directed Sampling} (\DIDS) to be a deterministic policy which at time $t$ chooses an action 
\begin{align}
x_t^\DIDS \in \argmin_{x \in \xX} \Psi_t^+(x) \text{ .} 
\end{align}
As the minimization happens over a subset of $\pP(\xX)$, clearly $\min_{\mu \in \pP(\xX)} \Psi_t^+(\mu) \leq \min_{x \in \xX}\Psi_t^+(x)$. Empirically, we found that which version has lower regret strongly dependents on the information function used to define IDS. We think that \DIDS\ could oftentimes yield a computational cheaper alternative to IDS with similar or sometimes better performance. Understanding the potential benefits of the randomized version is a task for future work.

\newpage

\section{Online Least Squares Estimation with Heteroscedastic Noise}\label{section: least squares regregression with heteroscedastic noise}
So far we have defined \emph{Information Directed Sampling} in terms of a surrogate $\Psi_t^+$ of the regret-information ration, which in turn relies on a high-probability upper bound $\Delta_t^+$ on the sub-optimality gaps. To construct this upper bound in the next section, we briefly discuss variants of classical least squares estimation for observations with heteroscedastic noise \cite[see also][]{AitkenIVleastsquares1936, RasmussenGaussianprocessesmachine2006}, and generalize the concentration inequalities of \cite{Abbasi-YadkoriOnlinelearninglinearly2012} for least squares estimators in $\RR^d$ and separable Hilbert spaces to the setting with heteroscedastic noise. Throughout this section, we assume that $\xX$ is a subset of $\RR^d$ and that the noise satisfies the subgaussian condition \eqref{eq: noise assumption}, which we summarize in a diagonal \emph{noise matrix} $\Sigma = \Sigma_T = \diag(\rho_1^2, \dots, \rho_T^2)$.

\subsection{Linear Least Squares with Heteroscedastic Noise}
Assume that the data $\{x_t,y_t\}_{t=1}^T$ is generated from a linear function $f(x) = x^\T \theta^*$ with unknown parameter $\theta^* \in \RR^d$. As common, we define the design matrix $X = X_T = (x_1,\dots, x_T)^\T$ and the observation vector $y = y_T = (y_1,\dots,y_T)^\T$, where we drop the subscript $T$ in favor of readability.

With heteroscedastic noise, the ordinary least squares estimator is known to be inefficient in general. \cite{AitkenIVleastsquares1936} was first to consider an alternative known as \emph{weighted linear least squares}, which is defined to minimize the weighted norm $\| X\theta - y \|_{\Sigma^{-1}}^2$. This amounts to applying ordinary least squares to normalized data $X' = \Sigma^{-1/2}X$ and $y' = \Sigma^{-1/2}y$, and has the wanted effect of putting less weight on noisier observation. Moreover, the Gauss-Markov theorem asserts that the estimator has the smallest possible variance among all linear and unbiased estimators, given a fixed design $X$ and Gaussian noise. Adding a Tikhonov regularization term depending a positive definite matrix $V_0$, we set $\hat{\theta}_T = \argmin_{\theta} \|y - \theta^\T X\|_{\Sigma^{-1}}^2 + \|\theta\|_{V_0}^2$, and with $V_T = X^\T\Sigma^{-1}X + V_0$, a closed form solution is given by $\hat{\theta}_T =  V_T^{-1}X^\T\Sigma^{-1} y$. The estimator $\hat{\theta}_T$ can also be motivated in the Bayesian setting with prior $p(\theta) \sim \nN(0, V_0^{-1})$ and likelihood $p(y|X,\theta) \sim \nN(X\theta, \Sigma)$, where $\hat{\theta}$ appears as mean of the posterior distribution $p(\theta|X,y) = \nN(\hat{\theta}_T, V_T^{-1})$.

In the next theorem, we extend the elegant concentration inequality for $\hat{\theta}$ from \cite{Abbasi-YadkoriImprovedAlgorithmsLinear2011} to the setting with heteroscedastic noise.
\begin{lemma}[Theorem 2 in \citep{Abbasi-YadkoriImprovedAlgorithmsLinear2011} for heteroscedastic noise]\label{lemma: concentration inequaliy for linear least squares}
	Let $(x_t, y_t)_{t\geq 1}$ be any stochastic process in $\RR^d \times \RR$ such that $y_t = x_t^\T \theta^* + \epsilon_t$ and $\epsilon_t$ satisfies noise assumption \eqref{eq: noise assumption}. Then, for $\hat{\theta}_T$ as defined above, with probability at least $1-\delta$, for any $T \geq 1$, it holds that
	\begin{align*}
	\|\hat{\theta}_T - \theta^* \|_{V_T} \leq \sqrt{2 \log\left(\frac{1}{\delta}\frac{\det(V_T)^{1/2}}{\det(V_0)^{1/2}}\right)} + \|\theta^*\|_{V_0} \text{ .}
	\end{align*}
\end{lemma}
\begin{proof}
	Theorem 2 in \citep{Abbasi-YadkoriImprovedAlgorithmsLinear2011}  proves the case $\rho_t=R$. We show that this immediately translates to the general case. For any $t \geq 1$, define $x_t' = \frac{x_t}{\rho_t}$, and corresponding observations $y_t' = \frac{y_t}{\rho_t}$. Now we apply the case $R=1$ to the process $\{x_t', y_t'\}$ and the corresponding estimator $\hat{\theta}_T' = V_T'^{-1}X'y'$. Observe that $\hat{\theta}_T = \hat{\theta}_T'$ and $V_T = V_T'$ to complete the proof.
\end{proof}
Note that by Cauchy-Schwarz, $|x^\T \hat{\theta}_T - x^\T \theta^*| \leq \|\hat{\theta}_T - \theta^*\|_{V_T}\|x\|_{V_T^{-1}}$ for all $x\in \xX$, and in combination with the previous lemma, this yields confidence intervals $\hat{f}_T(x) \pm\beta_T\sigma_T(x)$ that contain the true function value $f(x)$ with probability at least $1-\delta$ at any time $T$, where $\hat{f}_T(x) = x^\T\hat{\theta}_T$, $\beta_T = \sqrt{2 \log\left(\frac{1}{\delta}\frac{\det(V_T)^{1/2}}{\det(V_0)^{1/2}}\right)} + \|\theta^*\|_{V_0} $ and $\sigma_t(x) = \|x\|_{V_T^{-1}}$.
\subsection{Least Squares Estimation in Separable Hilbert Spaces}
Let $\hH$ be a separable Hilbert space with inner product $\<\cdot,\cdot\>_\hH$ and corresponding norm $\|\cdot\|_\hH$. For a positive definite operator $V_0 : \hH \rightarrow \hH$, we denote the inner product $\<\cdot, \cdot\>_{V_0} :=  \<\cdot, V_0\cdot\>$ with corresponding norm $\|\cdot \|_{V_0}$. Let $f^* \in \hH$ be the unknown function used to generate the data $\{v_t,y_t\}_{t=1}^T$, such that $y_t = \<v_t, f^*\>_\hH + \epsilon_t$. As substitute for the design matrix, define the operator $M = M_T : \hH \rightarrow \RR^T$, such that for all $ v\in \hH$ and $t = 1, \dots, T$, $(Mv)_t = \<v_t, v\>$, and denote its adjoint by $M^* : \RR^T \rightarrow \hH$. Taking the noise matrix $\Sigma = \Sigma_T = \diag(\rho_1^2, \dots, \rho_T^2)$ into account, we define the regularized least squares estimator $\hat{\mu}_T = \argmin_{f \in \hH}  \|M f - y\|_{\Sigma^{-1}}^2 + \|f\|_{V_0}^2$. Again, a closed form is given by $\hat{\mu}_T = V_T^{-1}M^*\Sigma^{-1}y$, where $V_T = M^*\Sigma^{-1} M + V_0$.

We now turn to concentration inequalities for $\hat{\mu}_T$ in separable Hilbert spaces as first introduced by \cite{Abbasi-YadkoriOnlinelearninglinearly2012}. The same inequality for the special case of a RKHS was recently re-derived by \cite{ChowdhuryKernelizedMultiarmedBandits2017}.
\begin{lemma}[Theorem 3.11 in \citep{Abbasi-YadkoriOnlinelearninglinearly2012} for heteroscedastic noise]\label{lemma: concentration inequality separaple hilbert spaces}
	Let $(v_t,y_t)_{t\geq 1}$ \linebreak be a stochastic process in $\hH \times \RR$, such that $y_t = \<v_t, f^*\>_\hH + \epsilon_t$ satisfies noise assumption \eqref{eq: noise assumption}. Then with probability at least $1-\delta$, for any $v \in \hH$ and for all $T \geq 1$,
	\begin{align*}
	|\<v,\hat{\mu}_T\> - \<v, f^*\>| \leq \left( \sqrt{ 2\log\left(\frac{1}{\delta}\frac{\det(\Sigma_T + M_TV_0^{-1}M_T^*)^{1/2}}{ \det(\Sigma_T)^{-1/2}}\right)} + \|f^*\|_{V_0}\right)  \|v\|_{V_T^{-1}} \text{ .}
	\end{align*}
\end{lemma}
\begin{proof}
	Theorem 3.11 in \cite{Abbasi-YadkoriOnlinelearninglinearly2012} proves the case $\rho_t=R$. Define $y_t' = y_t / \rho_t$ and $v_t' = v_t/\rho_t$, and apply the case $R=1$ to $\hat{\mu}'_T$. Note that $\hat{\mu}_T = \hat{\mu}'_T$ to complete the proof.
\end{proof}

\paragraph{Least Squares Estimation in RKHS}
Consider the important special case, where $\hH$ is a RKHS over $\RR^d$ with kernel function $k:\RR^d \times \RR^d \rightarrow \RR$, and canonical embeddings $k_x = k(x, \cdot)$ of $x \in \RR^d$. Let $v_t =k_{x_t} \in \hH$ be the embedding of $x_t \in \RR^d$, such that $\<v_t, f_*\>_\hH = \<k_{x_t}, f_*\>_\hH = f^*(x_t)$. Notably, for $V_0 = \lambda I$, the representer theorem yields the following, tractable form of $\hat{\mu}_T$,
\begin{align*}
\hat{\mu}_T(x) &= \<\hat{\mu}_T, k_x \>_\hH = k_T(x)^\T (K_T + \lambda \Sigma_T)^{-1}y \text{ ,}
\end{align*}
where we defined $(K_T)_{i,j} =  k(x_i,x_j)$ and $k_T(x)_t =  k(x_t,x)$. Details of this derivation are given in Appendix \ref{app: least squares estimation with heteroscedastic noise}. Similarly, we get  $\|v\|_{V_T^{-1}}^2 = \frac{1}{\lambda} \left(k(x,x) - k_T(x)^\T(K_T + \lambda \Sigma_T)^{-1}k_T(x)\right) =: \sigma_T(x)^2$.  Also note that $\hat{\mu}_T$ and $\sigma_T$ can be iteratively updated using Schur's complement \cite[Appendix F]{ChowdhuryKernelizedMultiarmedBandits2017}, avoiding the computation of a $T\times T$ matrix inverse at every step. In the next lemma, we specialize the previous result to the RKHS setting. 
\begin{lemma}\label{lemma: concentration inequalities for RKHS regression}
	Let $\hH$ be a RKHS, and $\hat{\mu}_T$ be the estimator defined above with $V_0 = \lambda I$ for $\lambda > 0$. Assume that the process $(x_t,y_t)_{t\geq 1}$ satisfies the noise assmption \eqref{eq: noise assumption}. Then, the following holds with probability at least $1-\delta$, for all $T \geq 1$, 
	\begin{align*}
	|\hat{\mu}_T(x) - f^*(x)| \leq  \left(\sqrt{2\log\left(\frac{1}{\delta}\frac{\det(\lambda \Sigma_T + K_T)^{1/2}}{\det(\lambda \Sigma_T)^{1/2}}\right)} + \sqrt{\lambda} \| f^*\|_{\hH}\right) \sigma_T(x) \text{ .}
	\end{align*}
\end{lemma}

We again draw the connection to Bayesian optimization. Let $f^* \sim \GP(0, \lambda^{-1} k)$ be a sample from a Gaussian process, where $\lambda > 0$ is a parameter trading off the magnitude of the prior variance, and assume that the observation noise $\epsilon_t$  independent and $\nN(0, \rho_t^2)$ distributed. Then the posterior distribution of $f^*$ is also a Gaussian process with mean $\hat{\mu}_T$ and covariance kernel\linebreak $k_T(x,y) = \frac{1}{\lambda}\left(k(x,y) - k_T(x)(K_T + \lambda \Sigma_T)^{-1}k_T(y)\right)$, and $\sigma_T(x)^2$ is the posterior variance at $x$.

\section{Regret Bounds for Information Directed Sampling on Heteroscedastic Bandits}\label{section: regret guarantees for ids}
We now define $\Delta_t^+$ for linear and RKHS functions $f$, using the regularized least squares estimator and the corresponding confidence intervals of Lemma \ref{lemma: concentration inequaliy for linear least squares} and \ref{lemma: concentration inequalities for RKHS regression}. In this case, we get estimators $\hat{f}_t : \xX \rightarrow \RR$ and a symmetric confidence interval $|f(x) - \hat{f}_t(x)| \leq \beta_t^\delta \sigma_t(x)$ at level $1-\delta$ for non-decreasing scale factor $\beta_t = \beta_t^\delta$ and confidence width $\sigma_t(x)$. We define the surrogate
\begin{align*}
\Delta_t^+(x) = \max_{x' \in \xX} \big(\hat{f}_t(x') + \beta_t \sigma_t(x')\big) - \big(\hat{f}_t(x) - \beta_t \sigma_t(x) \big) \text{ .}
\end{align*}
Note that $x_t^\UCB \in \argmax_{x \in \xX} \big(\hat{f}_t(x) + \beta_t \sigma_t(x)\big)$ is the UCB action and $\Delta_t^+(x_t^\UCB) = 2 \beta_t \sigma_t(x_t^\UCB)$. For simplicity, we now make the assumption that the noise function $\rho$ is known, and introduce two information gain functions $I_t^\F(x)$ and $I_t^\UCB(x)$, which will naturally lead to regret bounds for IDS.
\paragraph{IDS-F}
The first information function we propose is $I_t^\F = \log\big(1 + \frac{\sigma_t(x)^2}{\rho(x)^2}\big)$. We motivate \vspace{-2pt} this in the Bayesian setting with Gaussian prior and likelihood, where $I_t^\F$ is up to a constant factor equal to the \emph{conditional mutual information} $\II(f; x|\fF_{t-1})$. In the linear case, this follows by writing $\II(\theta; x_t,y_t|\fF_{t-1}) = \HH(\theta|\fF_{t-1}) - \HH(\theta|\fF_{t-1}, x_t,y_t)$, and then using the matrix determinant lemma,
\begin{align*}
2\,\II(\theta; x_t,y_t|\fF_{t-1}) = \log\big((2\pi e)^d \det(V_t^{-1})\big) - \log\big((2\pi e)^d \det(V_{t+1}^{-1})\big) =  \log\Big(1 + \tfrac{\sigma_t(x_t)^2}{\rho(x_t)^2}\Big) \text{ .}
\end{align*}
Note that the information gain $I_t(x)$ is diminishing as the uncertainty $\sigma_t(x)$ in $f$ decreases, allowing to bound $\gamma_T$. Importantly, $I_t(x)$ also characterizes the information gain depending on the noise, such that $I_t(x)$ decreases for larger $\rho(x)$. By compactness of $\xX$, we can also ensure that $I_t(x) \leq 1$. 

\paragraph{IDS-UCB}
Let $I_t^\UCB(x) = \log\left(\frac{\sigma_t(x_t^\UCB)^2}{\sigma_t(x_t^\UCB|x)^2}\right)$, where we define $\sigma_t(x_t^\UCB|x)$ as the  \vspace{-2pt} confidence width at $x_t^\UCB$ after $x$ has been evaluated. For the least squares estimators that we use in Section \ref{section: least squares regregression with heteroscedastic noise}, this can be computed without knowing the outcome of evaluating at $x$ and we explicitly calculate this quantity in Appendix \ref{app: information functions} . Again, $I_t^\UCB$ can be motivated in the Bayesian setting, where it corresponds to the mutual information $\II(x; f(x_t^\UCB)|\fF_{t-1})$. From the data processing inequality \cite[]{CoverElementsinformationtheory2012}, it further follows that $I_t^\UCB(x) \leq I_t^\F(x)$ for all $x \in \xX$. The resulting IDS-UCB policy collects information about the UCB action, but is not restricted to playing the same, if a more informative alternative action is available. Note also that both \DIDS-F and \DIDS-UCB are not more expensive to compute than standard UCB. We discuss further variants of IDS in Appendix \ref{app: information functions}.


To bound the regret of IDS-F and IDS-UCB, we will need the following lemma, which relates the regret-information ratio of IDS to the regret-information ratio of UCB.

\begin{lemma}\label{lemma: bounding the regret information surrogate}
	 At any time $t$, denote the sampling distribution of IDS with $\mu_t^\IDS$, the action chosen by \DIDS\ and UCB with $x_t^\DIDS$ and $x_t^\UCB$ respectively. Then, $\Psi_t^+(\mu_\IDS) \leq \Psi_t^+(x_\DIDS) \leq  \Psi_t^+(x_\UCB)$ holds. Moreover, if $I_t^\F$ or $I_t^\UCB$ is used,  $\Psi_t^+(x_\UCB) = \frac{8 \beta_t^2 \sigma_t(x_\UCB)^2}{ \log\left(1 + \sigma_t(x_\UCB)^2/\rho(x_\UCB)^2 \right)}$.
	 
	 If in this case also $I_t(x) \leq 1$ holds, then $\Psi_t^+(x_\UCB) \leq 8 \beta_t^2 \rho(x_\UCB)^2$, and consequently with $R = \max_{x \in \xX} \rho(x)$, for any of these strategies, $\sum_{t=1}^T \Psi_t  \leq  8 \beta_T^2 R^2 T$. 
\end{lemma}
\begin{proof}
	The sequence of inequalities follows directly from the respective definitions. The formula for $\Psi_t^+(x_\UCB)$ follows by noting that $I_t^F(x_t^\UCB) = I_t^\UCB(x_t^\UCB) = \log\left(1 + \frac{\sigma_t(x_\UCB)^2}{\rho(x_\UCB)^2} \right)$. \vspace{-4pt} If $I_t(x) \leq 1$ for all $x \in \xX$, we can use $2 \log\left(1 + \frac{\sigma_t(x)^2}{\rho(x)^2}\right) \geq \frac{\sigma_t(x)^2}{ \rho(x)^2}$, which gives$\Psi_t^+(x_\UCB) \leq 8 \beta_t^2 \rho(x_\UCB)^2$. The last inequality uses that $\beta_t$ is non-decreasing, and uniformly bounds $\rho(x) \leq R$.
\end{proof}

Finally, we obtain regret bounds for IDS and \DIDS\ using the previous lemma, Theorem \ref{thm: general regret bound} or \ref{thm: general regret bound deterministic}, and the concentration inequalities in Lemma \ref{lemma: concentration inequaliy for linear least squares} and \ref{lemma: concentration inequalities for RKHS regression}.
\begin{corollary} Denote $\xX_t^\UCB$ the set of UCB actions at time $t$ as observed during the run of IDS and DIDS respectively, and set $R = \sup_{t\geq 1} \min_{x \in \xX_t^\UCB} \rho(x)$. For estimators $\hat{f}_t$ with confidence bounds $\hat{f}_t\, \pm\, \beta_t^\delta\sigma_t(x)$, with probability at least $1-\delta$, the regret of IDS-F and IDS-UCB is bounded by $R_T = \oO\left( R \beta_T^\delta \sqrt{T (\gamma_T +\log \frac{1}{\delta})} \right)$, and of \DIDS-F and \DIDS-UCB by $R_T = \oO\Big(R \beta_T^\delta \sqrt{T \gamma_T}\Big)$. 
\end{corollary}

\begin{table}[]
	\centering
	\begin{tabular}{lllll}
	$f$ 	& $\beta_T^\delta$                                            & $\gamma_T$                   & $R_T$                                                                                                                                                  \\ \hline
		linear                & $\oO\Big(\sqrt{d\log(T) + \log(\frac{1}{\delta})}\Big)$               & $\oO\big(d\log(T)\big)$     & $\oO\left( (d\log(T) + \log\frac{1}{\delta})\sqrt{T}\right)$                                        \\
		\begin{tabular}[c]{@{}l@{}}RKHS with\\ RBF kernel\end{tabular}  & $\oO\left(\sqrt{\log(T)^d + \log(\frac{1}{\delta})}\right)$ & $\oO\left(\log(T)^d\right)$ & $\oO\left((\log(T)^d + \log(\frac{1}{\delta})) \sqrt{T} \right)$ 
	\end{tabular}
	\caption{Summary of worst case regret bounds for IDS-F/IDS-UCB. For the bounds on $\beta_T^\delta$ and  $\gamma_T$, refer to \cite[]{Abbasi-YadkoriImprovedAlgorithmsLinear2011, SrinivasGaussianProcessOptimization2010}, including bounds for further kernels. }
	\label{tbl:regret bounds}
\end{table}

We summarize the resulting regret bounds in Table \ref{tbl:regret bounds}. Note that in the homoscedastic setting, we recover known bounds for UCB via Lemma \ref{lemma: bounding the regret information surrogate}. In settings with heteroscedastic noise, our bounds can be much better, for example as depicted in Example \ref{exa: ucb and ts fail}, where the noise constant $R$, as defined in the Corollary, is smaller for IDS than for UCB. To obtain bounds adaptive to more general heteroscedastic settings, one would need to bound the sum $\sum_{t=1}^T \Psi_t$ directly in relation to $\rho$ and the policy; and general instance depending bounds are a significant undertaking even for linear functions \citep{LattimoreEndOptimismAsymptotic2017}. As our bounds are \emph{worst-case guarantees}, the actual instance-dependent regret can be rather different; see also the experiments in the next section.

\section{Experiments} \label{section: experiments}
We show a simple, synthetic experiment, where we fix a linear function in $\RR^3$ and for each experiment randomly generated a set of 30 actions contained in the Euclidean unit ball. In the homoscedastic setting (Figure \ref{fig: homoscedastic_A}) we use Gaussian noise $\nN(0, \rho^2)$ with constant noise bound $\rho=0.5$, and compare to UCB and TS. We define TS directly using the posterior, and do not employ oversampling as needed for the frequentist guarantees in \cite{AgrawalThompsonSamplingContextual2013}. In the heteroscedastic setting (Figure \ref{fig: heteroscedastic_C}), we chose the bound for each action uniformly in $[0.1, 1]$, and compare to UCB and TS, which use $1$ as an upper bound on the noise, and to W-UCB and W-TS, which use the weighted least squares estimator and Lemma \ref{lemma: concentration inequaliy for linear least squares}. 
For our methods, we show D/IDS-F, IDS-UCB and IDS-TS, where the latter policy is defined by $I_t^\TS(x) = \II(x;f(x_t^\TS)|\fF_{t-1})$ with respect to a \emph{proposal action} $x_t^\TS$ from Thomson Sampling; but we currently do not have regret bounds for IDS-TS. We omit the regret curve of DIDS-UCB and DIDS-TS as it was too close to the respective randomized version. For all methods, we calculate confidence bounds from Lemma \ref{lemma: concentration inequaliy for linear least squares} directly using the determinant and the matrix $V_T$, and we fix $\delta=0.01$, $\lambda=1$. Error bars show 2 times standard error of 500 trials. We observe that our methods stay competitive with UCB/TS in the homoscedastic setting, expect for IDS-F, which is over-exploring, likely because $I_t^F$ captures global information about $f$ irrespective of the location of the optimum. With heteroscedastic noise, IDS-UCB/IDS-TS show an advantage over UCB/TS, demonstrating the importance of using the noise both for estimation and exploration.
Further experiments, including a simulation of Example \ref{exa: ucb and ts fail}, can be found in Appendix \ref{app: further experiments}.\pagebreak

\begin{figure}
	\centering   	
	\hspace{-10pt}
	\subfigure[Homoscedastic Noise][t]{\label{fig: homoscedastic_A}\includegraphics[scale=1,trim={0 0 0 0},clip]{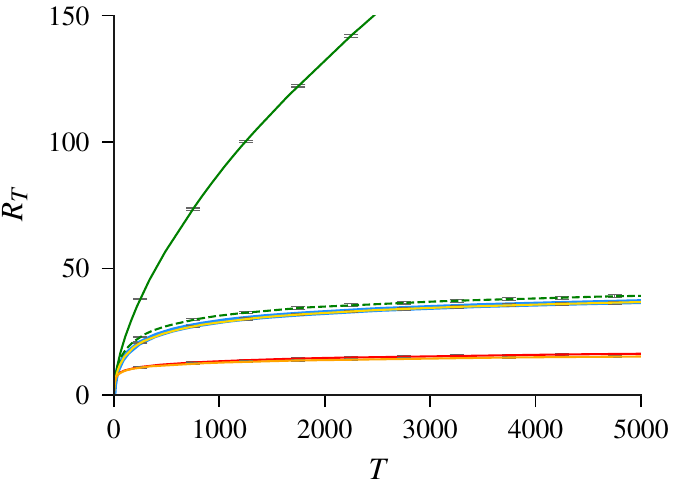}}
\hspace{-5pt}
	\subfigure[Heteroscedastic Noise][t]{\label{fig: heteroscedastic_C}\includegraphics[scale=1,trim={0.5cm 0 0 0},clip]{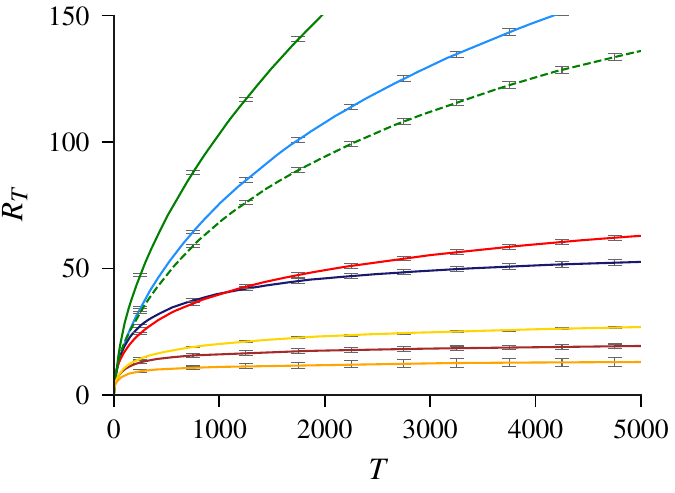}}~
	\includegraphics[scale=1,trim={130pt 0pt 10pt 0},clip]{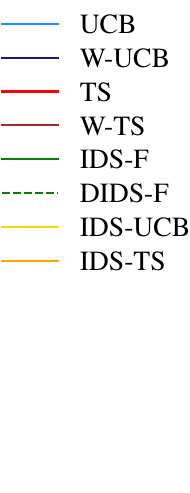}
	\hspace{-10pt}
	\caption{With homoscedastic noise, IDS-UCB and UCB achieve almost the same regret, and the same holds for IDS-TS and TS. Interestingly, while DIDS-F performs similar to IDS-UCB, randomization allows IDS-F to explore more; numerically, this regret curve is close to the theoretical bound $\oO(\sqrt{T\log T})$. In the heteroscedastic experiment, IDS-UCB is the best method with known frequentist guarantees, outperforming UCB/W-UCB, and similarly IDS-TS improves on TS/W-TS. This  demonstrates the importance of using the noise bound both for estimation and exploration.}
\end{figure}

\section{Conclusion and Future Work}\label{section: future work} 
We introduced a frequentist version of IDS, and use it to address bandits with heteroscedastic noise. However, we believe that the generality of our analysis opens up many promising directions for future work, some of which we outline below. As termed by \cite{RussoLearningOptimizeInformationDirected2014}, \emph{Information Directed Sampling} should be rather seen as a \emph{design principle} than an algorithm, as it depends on the user-defined choices of $I_t$ and $\Delta_t^+$. Empirically, we  observed that the performance of IDS strongly depends on the choice of $I_t$, hence better understanding and deriving new information functions is an important task for future work. For instance, \cite{RussoLearningOptimizeInformationDirected2014} use \linebreak $I_t(x) = \II(x; \argmax_{x \in \xX} f(x)|\fF_{t-1})$, but this is in general more expensive to compute and one might have to use approximation techniques, as for example done by \cite{WangMaxvalueEntropySearch2017} for $I_t(x) = \II(x; \max_{x \in \xX} f(x)|\fF_{t-1})$. The choice of $I_t$ can also be seen as introducing prior information on the problem structure to the algorithm, which could be useful in other settings as well.

Regarding bandits with heteroscedastic noise, we assumed that the heteroscedastic noise bound $\rho$ is known. For applications one might instead want to estimate the variance of the noise directly from the data, as done by \cite{CowanNormalBanditsUnknown2015} for Gaussian multi-armed bandits. We also point out a striking similarity of IDS-F and an instance-dependent lower bound, due to \cite{BurnetasOptimaladaptivepolicies1996}, see equation (8) in \citep{CowanNormalBanditsUnknown2015}. Finally, it needs to be seen if the notion of regret-information ratio is useful beyond the algorithms we study here. For example, it might be possible to bound the regret-information ratio of Thompson Sampling, where regret bounds in the frequentist setting remain notoriously harder to achieve than in the Bayesian setting.

\acks{We thank Kfir Levy for helpful discussions on this paper and for pointing out Lemma \ref{lemma: kakade_cocentration_lemma}. This research was supported by SNSF grant  200020\_159557.}
\newpage

\bibliography{references.bib}

\appendix

\section{Concentration Inequalities for Martingales}\label{app: concentration inequalites for martingales}
We first state the following consequence of Freedman's inequality.
\begin{lemma}[Similar to Lemma 3 in \citep{KakadeGeneralizationAbilityOnline2009}]\label{lemma: kakade_cocentration_lemma} Let $X_1, \dots, X_T$ be a martingale difference sequence with respect to a filtration $\fF_t$ such that $X_t \leq b$ holds for all $t = 1, \dots, T$. Denote the corresponding martingale by $M_T = \sum_{t=1}^T X_t$ and the sum of conditional variances by $V_T = \sum_{t=1}^T \EE[X_t^2|\fF_{t-1}]$. Then, for any $\beta  > 1$, $l = \left\lceil \frac{\log (T \lambda^{-2})}{\log \beta} \right \rceil$, $\lambda \geq 0$,
	\begin{align*}
		\PP[M_T \geq \lambda \max\{\lambda b, \sqrt{V_T}\}] \leq (l+1) \exp\left(- \frac{\lambda^2}{2\beta + \frac{2}{3}}\right) \text{ ,}
	\end{align*}
	and consequently, with probability at least $1 -  \delta$,
	\begin{align}
	M_T \leq \max \left\{ 4 b \log \frac{2T + 2}{\delta}, 2 \sqrt{V_T \log \frac{2T + 2}{\delta}} \right\}
	\end{align}
\end{lemma}
\begin{proof} Our proof uses the same ideas as in the proof of Lemma 3 of \cite{KakadeGeneralizationAbilityOnline2009} and is included for completeness; along the way we also get rid of the assumptions $\delta < 1/e$ and $T \geq 3$ used in the original lemma. Define $l = \left\lceil \frac{\log T \lambda^{-2}}{\log \beta} \right \rceil$, and $\alpha_i = \lambda^2 b^2 \beta^i$ for $i=0, \dots, l$, further set $\alpha_{-1} = 0$ for notational convenience. Note that since $V_T \leq T b^2$, our choice of $l$ implies $\alpha_l \geq V_T$. Then,
	\begin{align*}
		&\+\PP[M_T \geq \lambda \max\{\lambda b, \sqrt{V_T}\}]\\
		& = \PP[M_T \geq \lambda \max\{\lambda b, \sqrt{V_T}\},\; \alpha_{i-1} \leq V_T \leq \alpha_i \text{ for } i = 0, \dots, l ]\\
		& \leq \sum_{i=0}^l \PP[M_T \geq \lambda\max\{\lambda b, \sqrt{  V_T} \},\; \alpha_{i-1} \leq V_T \leq \alpha_i]\\
		& \leq \sum_{i=0}^l \PP[M_T \geq \lambda \max\{\lambda b, \sqrt{\alpha_{i-1}}\},\;  V_T \leq \alpha_i]\\
		&\qeqlabel{\leq}{qeq: kt-1} \sum_{i=0}^l \exp\left(- \frac{\lambda^2 \max\{\lambda^2 b^2, \alpha_{i-1}\}}{2\alpha_i + \frac{2}{3}b \lambda \max\{\lambda b, \sqrt{\alpha_{i-1}}\}}\right) \\
		&= \exp\left(-\frac{\lambda^2}{2 + \frac{2}{3}}\right) + \sum_{i=1}^l \exp\left(- \frac{\lambda^2}{2 \beta + \frac{2}{3} \sqrt{\frac{1}{\beta^{i-1}}} }\right)\\
		&\qeqlabel{\leq}{qeq: kt-2} (l+1)\exp\left(-\frac{\lambda^2}{2\beta + \frac{2}{3}}\right) \text{ .}
	\end{align*}
	Here, \qeqref{qeq: kt-1} is Freedman's inequality \cite[]{Freedmantailprobabilitiesmartingales1975}, and \qeqref{qeq: kt-2} uses that $\beta > 1$. This shows the first part of the lemma. For the second part, note that for $T = 1$, the claim is trivially true by $X_1 \leq b$, hence we assume $T \geq 2$. If also $\lambda \geq 1$, we can upper bound $l \leq \left\lceil \frac{\log T}{\log \beta} \right \rceil =: l'$. We  choose $\beta = 5/3$, and set $\lambda = \sqrt{\log \frac{l' + 1}{\delta} (2 \beta + \frac{2}{3})} = 2 \sqrt{\log \frac{l' + 1}{\delta}}$, such that indeed $\lambda \geq 2 \sqrt{\log\left(\frac{\log(2)}{\log(5/3)} + 1 \right) }\geq 1$ for $T \geq 2$, hence proving the claim.
\end{proof}

The next concentration inequality for supermartingales is an anytime variant of Corollary 2.7 by \cite{FanExponentialinequalitiesmartingales2015}, and might be of independent interest.
\begin{theorem}\label{thm: any time fan concentration} Let $S_T = \sum_{t=1}^T \xi_t$ be a sum of supermartingale differences $\xi_t$ with filtration $\{\fF_t\}$. Further, let $U_t$ be a non-negative predictable process, such that $\xi_t \leq U_t$ holds for all $t \geq 1$. Define
	\begin{align*}
	C_t^2 =	\begin{cases} \EE[\xi_t^2|\fF_{t-1}], & \text{if } \EE[\xi_t^2|\fF_{t-1}] \geq U_t^2\;,\\
	\frac{1}{4}\left(U_t + \frac{\EE[\xi_t^2|\fF_{t-1}]}{U_t}\right)^2 & \text{otherwise.}
	\end{cases}
	\end{align*}
	Further denote $A_T = \sum_{t=1}^T C_t^2$. Then, for any fixed positive sequence $(l_t)_{t=1}^T$, with probability at least $1-\delta$, 
	\begin{align*}
	\forall\; T \geq 1, \quad S_T \leq \sqrt{2(A_T + l_T) \log\left(\frac{1}{\delta}\frac{(A_T + l_T)^{1/2}}{l_T^{1/2}}\right)} \text { .}
	\end{align*}
\end{theorem}
\noindent
The proof combines Corollary 2.7 in \citep{FanExponentialinequalitiesmartingales2015}, with the \emph{method of mixtures} as used in the proof of Theorem 1 in \citep{Abbasi-YadkoriImprovedAlgorithmsLinear2011}. We start with the following lemma.
\begin{lemma}\label{lemma: supermartingale for fan concentration}
	Let $\xi_t$, $C_t$ as in Theorem \ref{thm: any time fan concentration}, and define for $\lambda \geq 0$, $t \geq 1$,
	\begin{align}
	M_t^\lambda = \exp\left(\sum_{s=1}^t \lambda \xi_s - \frac{\lambda^2}{2}C_s^2\right) \text{ .}
	\end{align}
	Further, let $\tau$ be a stopping time with respect to the filtration $\{\fF_t\}$. Then $M_t^\lambda$ is a supermartingale, $M_\tau^\lambda$ is almost surely well-defined, and $\EE[M_\tau^\lambda] \leq 1$.
\end{lemma}
\begin{proof} The proof is along the lines of Lemma 8 in \citep{Abbasi-YadkoriImprovedAlgorithmsLinear2011}, where we replace the subgaussian condition by the suitable analog to showing that $M_t^\lambda$ is a supermartingale. Let
	\begin{align*}
	D_s^\lambda = \exp\left(\lambda \xi_s - \frac{\lambda}{2} C_s^2\right) \text{ .}
	\end{align*}
	By Corollary 2.6 of \cite{FanExponentialinequalitiesmartingales2015}, we have that $\EE[e^{\lambda\xi_s} |\fF_{s-1}] \leq \exp\left(\tfrac{\lambda^2}{2}C_s^2\right)$ for all $\lambda > 0$, and consequently, $\EE[D_s|\fF_{s-1}] \leq 1$. Therefore, $\EE[M_t^\lambda |\fF_{t-1}] = \EE[D_t^\lambda|\fF_{t-1}] M_{t-1}^\lambda \leq M_{t-1}^\lambda$, which shows that $M_t^\lambda$ is a supermartingale such that $\EE[M_t^\lambda] \leq 1$ for all $t \geq 1$. The rest of the argument is standard \cite[compare for example Theorem 5.7.6 in][]{DurrettProbabilitytheoryexamples2010}; by the convergence theorem for non-negative supermartingales, $M_\tau^\lambda$ is well defined for any stopping time $\tau \leq \infty$, then using Fatou's lemma it follows that $M_\tau^\lambda \leq \liminf_{t \rightarrow \infty} M_{\tau \wedge t}^\lambda \leq 1$. 
\end{proof}

To prove Theorem \ref{thm: any time fan concentration}, we now use the \emph{method of mixtures}, similar to Theorem 1 in \citep{Abbasi-YadkoriImprovedAlgorithmsLinear2011}. The main difference is that the supermartingale $M_t^\lambda$ from the the previous lemma is only defined for $\lambda \geq 0$, which requires to choose a mixing density supported on $[0, \infty)$.

\begin{proof}\textbf{of Theorem \ref{thm: any time fan concentration}}\ 
	Remember that $S_t = \sum_{s=1}^t \xi_s$, $A_t =\sum_{s=1}^t C_s^2$ and $M_t^\lambda = \exp\left(\lambda S_t - \tfrac{\lambda^2}{2} A_t\right)$.
	Further, let $\Lambda = (\Lambda_t)_{t\geq 1}$ be a sequence of independent Gaussian random variable truncated to $\RR^+ = [0,\infty)$ with densities $f_{\Lambda_t}(\lambda) =  c(l_t)\exp\left(-\tfrac{1}{2}\lambda^2 l_t\right) \I{\lambda \geq 0}$ where $c(A) = \sqrt{\frac{2 A}{\pi}}$ is a normalizing constant. Using $\Lambda$ as a \emph{mixing distribution} we define
	\begin{align}
	M_t = \EE[M_t^{\Lambda_t}|\fF_\infty] \text{ ,}
	\end{align}
	where $\fF_\infty = \sigma\left(\cup_{t=1}^\infty \fF_t\right)$ is the tail $\sigma$-algebra of the filtration $\fF_t$. In particular, using Fubini's theorem, we still get $\EE[M_\tau] = \EE[\EE[M_\tau^{\Lambda_t}|\Lambda]] \leq 1$. In the next step, we explicitly calculate $M_t$ for any $t \geq 1$,
	\begin{align*}
	M_t &= \int_{\RR^+}\exp\left(\lambda S_t - \tfrac{\lambda^2}{2}A_t\right)f_{\Lambda_t}(\lambda)\; d\lambda\\
	&= \int_{\RR^+}\exp\left(-\frac{1}{2}\left(\lambda -\frac{ S_t }{A_t}\right)^2 A_t + \frac{1}{2}\frac{S_t^2}{A_t}\right)f_{\Lambda_t}(\lambda)\; d\lambda\\
	&= \exp\left(\frac{1}{2}\frac{S_t^2}{A_t}\right)\int_{\RR^+}\exp\left(-\frac{1}{2}\left(\lambda -\frac{ S_t }{A_t}\right)^2 A_t\right)f_{\Lambda_t}(\lambda)\; d\lambda\\
	&= c(l_t)\exp\left(\frac{1}{2}\frac{S_t^2}{A_t}\right)\int_{\RR^+}\exp\left(-\frac{1}{2}\left(\left(\lambda - S_t /A_t\right)^2 A_t + \lambda^2l_t\right)\right)\; d\lambda \text{ .}
	\end{align*}
	Completing the square yields
	\begin{align*}
	\left(\lambda - \frac{S_t}{A_t}\right)^2 A_t + \lambda^2A= \left(\lambda - \frac{S_t}{l_t + A_t}\right)^2 (l_t + A_t) + \frac{S_t^2}{A_t} - \frac{S_t^2}{l_t + A_t} \text{ ,}
	\end{align*}
	and with the previous equation,
	\begin{align*}
	M_t &= c(l_t)\exp\left(\frac{1}{2}\frac{S_t^2}{l_t + A_t}\right)\int_{\RR^+}\exp\left(-\frac{1}{2}\left(\left(\lambda - S_t /(l_t + A_t)\right)^2 (l_t + A_t)\right)\right)\; d\lambda\\
	&\qeqlabel{\geq}{qeq: mm-A-1} \I{S_t \geq 0} c(l_t)\exp\left(\frac{1}{2}\frac{S_t^2}{l_t + A_t}\right)\int_{\RR^+}\exp\left(-\frac{1}{2}\left(\left(\lambda - S_t /(l_t + A_t)\right)^2 (l_t + A_t)\right)\right)\; d\lambda \\
	&\qeqlabel{\geq}{qeq: mm-A-2} \I{S_t \geq 0} c(l_t)\exp\left(\frac{1}{2}\frac{S_t^2}{l_t + A_t}\right)\int_{\RR^+}\exp\left(-\frac{1}{2}\left(\lambda^2 (l_t + A_t)\right)\right)\; d\lambda \\
	&= \I{S_t \geq 0}\frac{c(l_t)}{c(l_t + A_t)}\exp\left(\frac{1}{2}\frac{S_t^2}{l_t + A_t}\right) \text{ .}
	\end{align*}
	In \qeqref{qeq: mm-A-1} we introduced an indicator function and used that all other terms are positive. To get \qeqref{qeq: mm-A-2}, we first applied a change of variables $\lambda' = \lambda - S_t /(l_t + A_t)$ and then made use of $S_t \geq 0$ to reduce the integration range (and again that the integrand is positive).
	
	\noindent
	A final application of Markov's inequality yields
	\begin{align*}
	&\+\PP\left[S_\tau \geq \sqrt{2 (l_\tau + A_\tau) \log\left(\frac{1}{\delta}\frac{(l_\tau + A_\tau)^{1/2}}{l_\tau^{1/2}}\right)}\right] \\
	&= \PP\left[\frac{c(l_\tau)}{c(l_\tau + A_\tau)}\exp\left(\frac{1}{2}\frac{S_\tau^2}{l_\tau + A_\tau}\right) \geq \frac{1}{\delta} , S_\tau \geq 0\right] \\
	&\leq \delta \cdot \EE\left[\I{S_\tau \geq 0}\frac{c(l_\tau)}{c(l_\tau + A_\tau)}\exp\left(\frac{1}{2}\frac{S_\tau^2}{l_\tau + A_\tau}\right)\right]\\
	&\qeqlabel{\leq}{qeq: mm-B-1} \delta \cdot \EE[M_\tau] \qeqlabel{\leq}{qeq: mm-B-2} \delta \text{ ,}
	\end{align*}
	where \qeqref{qeq: mm-B-1} uses the inequality for $M_t$ derived above, and \qeqref{qeq: mm-B-2} follows from Lemma \ref{lemma: supermartingale for fan concentration}. 
	
	To get the anytime result as stated in the Theorem, we use the same argument as in \citep{Abbasi-YadkoriImprovedAlgorithmsLinear2011} on the stopping time
	\begin{align*}
	\tau = \min\left\{t \geq 1 \;|\; S_t \geq \sqrt{2 (l_t + A_t) \log\left(\frac{1}{\delta}\frac{(l_t + A_t)^{1/2}}{l_t^{1/2}}\right)} \right\} \text{ .}
	\end{align*}
	Expressing the quantity of interest in terms of $\tau$, and applying the previous inequality yields
	\begin{align*}
	&\+ \PP\left[ S_t \geq \sqrt{2 (l_t + A_t) \log\left(\frac{1}{\delta}\frac{(l_t + A_t)^{1/2}}{l_t^{1/2}}\right)} \text{ for any } t \geq 1\right] \\
	&=	\PP[\tau < \infty] \\
	&= \PP\left[\tau < \infty,\, S_\tau \geq \sqrt{2 (l_\tau + A_\tau) \log\left(\frac{1}{\delta}\frac{(l_\tau + A_\tau)^{1/2}}{l_\tau^{1/2}}\right)}\right] \\
	&\leq \PP\left[S_\tau \geq \sqrt{2 (l_\tau + A_\tau) \log\left(\frac{1}{\delta}\frac{(l_\tau + A_\tau)^{1/2}}{l_\tau^{1/2}}\right)} \right]\\
	&\leq \delta \text{ .}
	\end{align*}
	This completes the proof.
\end{proof}
\noindent
As a consequence of the previous result, we have the following lemma.
\begin{lemma*}{\textbf{\ref{lemma: concentration of conditinal means}}} {(in Section \ref{section: general regret framework})} 
	Let $X_t$ be a non-negative stochastic process adapted to a filtration $\{\fF_t \}$. Assume that $X_t \leq b_t$ for a fixed, non-decreasing sequence $\{b_t\}$. Define $m_t = \EE[X_t|\fF_{t-1}]$ and $M_T = \sum_{t=1}^T m_t$, and let $(l_t)_{t\geq 1}$ be any fixed, positive sequence. Then, with probability at least $1-\delta$,
	\begin{align*}
	\forall \; T\in \NN, \quad \sum_{t=1}^T m_t - X_t &\leq \sqrt{2(b_T M_T + l_T) \log\left(\frac{1}{\delta} \frac{(b_T M_T + l_T)^{1/2}}{l_T^{1/2}}\right)} \text{ .}
	\end{align*}
	Further, if $b_T \geq 1$, with probability at least $1-\delta$ for any $T \geq 1$ it holds that,
	\begin{align*}
	\sum_{t=1}^T m_t \leq 2\sum_{t=1}^T X_t + 4 b_T \log\frac{1}{\delta} + 8  b_T\log(4 b_T) + 1 \text{ .}
	\end{align*}
\end{lemma*}

\begin{proof}
	Clearly, $\xi_t = m_t - X_t$ is a martingale difference sequence such that $\xi_t \leq m_t$, and $m_t$ is a predictable process. This allows us to apply Theorem \ref{thm: any time fan concentration}. Remember that
	\begin{align*}
	C_t^2 =	\begin{cases} \EE[\xi_t^2|\fF_{t-1}], & \text{if } \EE[\xi_t^2|\fF_{t-1}] \geq m_t^2\;,\\
	\frac{1}{4}\left(m_t + \frac{\EE[\xi_t^2|\fF_{t-1}]}{m_t}\right)^2 & \text{otherwise,}
	\end{cases}
	\end{align*} 
	and in particular $C_t^2 \leq b_t m_t$. To see this, note that the Bhatia-Davis inequality (Lemma \ref{lemma: bhatia-davis}, Appendix \ref{app: useful inequalities}) implies $\EE[\xi_t^2|\fF_{t-1}] = \Var(X_t|\fF_{t-1}) \leq m_t (b_t - m_t) \leq m_t b_t$, and if $\EE[\xi_t^2|\fF_{t-1}] \leq m_t^2$, and further, if $m_t \leq b_t$ hold, $\frac{1}{4}\left(m_t + \frac{\EE[\xi_t^2|\fF_{t-1}]}{m_t}\right)^2 \leq m_t^2 \leq m_t b_t$. Consequently, $A_T = \sum_{t=1}^T C_t^2 \leq b_T \sum_{t=1}^T m_t = b_T M_T$. A direct application of Theorem \ref{thm: any time fan concentration} shows the first inequality. 
	
	For the second part, we use a trade-off parameter $\eta > 1$, and set $l_t = b_t$ to find 
	\begin{align*}
	&\+\sqrt{2(b_T M_T + b_T) \log\left(\frac{1}{\delta} \frac{(b_T M_T + b_T)^{1/2}}{b_T^{1/2}}\right)}\\
	&\qeqlabel{\leq}{qeq:ccm-1} \sqrt{2 b_T (M_T + 1)\log\left(\frac{1}{\delta}\right)} + \sqrt{b_T(M_T + 1) \log(M_T + 1)}\\
	&\qeqlabel{\leq}{qeq:ccm-2} \frac{M_T}{2\eta} + \frac{1}{2 \eta} + \eta b_T \log\left(\frac{1}{\delta}\right) + \frac{M_T}{2\eta} + \frac{1}{2\eta} + 2\eta b_T \log(2 \eta b_T)
	\end{align*}
	where for \qeqref{qeq:ccm-1} we used the inequality $\sqrt{a + b} \leq \sqrt{a} + \sqrt{b}$, and for \qeqref{qeq:ccm-2} we used $2\sqrt{ab} \leq a + b$ (Lemma \ref{lemma: 2 sqrt ab < a + b}) on the first term, and $\sqrt{ab \log(ab)} \leq a + b\log(b)$ (Lemma \ref{lemma: sqrt ab ln(ab) < a + b ln(b)}) on the second term with the condition $b_T \geq 1$.	We set $\eta=2$ and combine the previous inequalities. Solving for $M_T=\sum_{t=1}^T m_t$ yields
	\begin{align*}
	\sum_{t=1}^T m_t \leq 2 \sum_{t=1}^T X_t + 4 b_T \log\left(\frac{1}{\delta}\right) + 8  b_T \log(4 b_T)  + 1 \text{ .}
	\end{align*}
	This completes the proof.
\end{proof}

\section{Proof of Theorem \ref{thm: general regret bound}}\label{app: proof of main theorem}

We will write $\Delta_t = \Delta(x_t)$ and $I_t = I_t(x_t)$ for notational convenience.
\paragraph{Step 1) Bounding the regret by the expected regret.} 
We first use Lemma \ref{lemma: kakade_cocentration_lemma} to bound the regret in terms of the sum of expected instantaneous regret $\sum_{t=1}^T \EE[\Delta_t|\fF_{t-1}]$.
\begin{lemma}\label{lemma: bound regret by instantaneous regret}	
	For any policy $\pi$ we have with probability at least $1-\delta$,
	\begin{align*}
	R_T &\leq \tfrac{5}{4}  \sum_{t=1}^T \EE[\Delta_t|\fF_{t-1}] + 4 S \log\left(\tfrac{2 \log(T) + 2}{\delta}\right) \text{ .}
	\end{align*}
	Further, with probability at least $1-\delta$, at any time $T \geq 1$,
	\begin{align*}
	R_T &\leq \tfrac{5}{4}  \sum_{t=1}^T \EE[\Delta_t|\fF_{t-1}] + 4 S \log\left(\tfrac{4\pi^2 T^2}{3\delta} (\log(T) + 1)\right) \text{ .}
	\end{align*}
\end{lemma}
\begin{proof}
	First, we rewrite the regret 
	\begin{align*}
	R_T &= \sum_{t=1}^T \Delta_t =  \sum_{t=1}^T \EE[\Delta_t|\fF_{t-1}] + \sum_{t=1}^T \left( \Delta_t -  \EE[\Delta_t|\fF_{t-1}] \right) \text{ ,}
	\end{align*}	
	such that we only need to bound the martingale difference sequence $D_t = \Delta_t -  \EE[\Delta_t|\fF_{t-1}]$. Note that by assumption $0 \leq \Delta_t \leq S$, hence, $D_t \leq S$. Further, by the Bhatia-Davis inequality,
	\begin{align*}
	\Var[D_t|\fF_{t-1}] = \Var[\Delta_t|\fF_{t-1}] \leq \EE[\Delta_t |\fF_{t-1}](S- \EE[\Delta_t |\fF_{t-1}]) \text{ ,}
	\end{align*}
	and consequently,
	\begin{align*}
	V_T = \sum_{t=1}^T \Var[D_t|\fF_{t-1}] \leq S \sum_{t=1}^T \EE[\Delta_t |\fF_{t-1}] \text{ .}
	\end{align*}
	Now we apply Lemma \ref{lemma: kakade_cocentration_lemma} to find that with probability at least $1-\delta$,
	\begin{align*}
	\sum_{t=1}^T D_t & \leq \max\left\{4S \log\tfrac{2 \log(T) + 2}{\delta}, 2 \sqrt{V_T \log \tfrac{2 \log(T) + 2}{\delta}}\right\}\\
	&\leq \max\left\{4S \log\tfrac{2 \log(T) + 2}{\delta}, 2 \sqrt{S \sum_{t=1}^T \EE[\Delta_t | \fF_{t-1}] \log \tfrac{2 \log(T) + 2}{\delta}}\right\}\\
	&\leq \frac{1}{4} \sum_{t=1}^T \EE[\Delta_t | \fF_{t-1}] + 4 S\log \tfrac{2 \log(T) + 2}{\delta}
	\end{align*}
	where for the last inequality we have used $2 \sqrt{ab}\leq a + b$ for $a,b \geq 0$. This proves the first part of the lemma. Taking the union bound over all time steps $T \geq 1$, effectively replacing $\delta$ by $\frac{6 \delta}{\pi^2 T^2}$, gives the anytime bound.
\end{proof}

\paragraph{Step 2) Application of Cauchy-Schwarz.}
In the next lemma, we bound the sum of expected instantaneous regret using the information-regret ratio and the expected information gain. 
\begin{lemma} \label{lemma: expected regret by information ratio}
	\begin{align*}
	\sum_{t=1}^T \EE[\Delta_t|\fF_t] &= \sum_{t=1}^T \sqrt{\Psi_t} \sqrt{\EE[I_t |\fF_{t-1}]} \leq \sqrt{\sum_{t=1}^T \Psi_t \cdot \sum_{t=1}^T \EE[I_t|\fF_{t-1}]} \text{ .}
	\end{align*}
\end{lemma}
\begin{proof}
	First use the definition of $\Psi_t$ and then apply Cauchy-Schwarz.
\end{proof}

\paragraph{Step 3) Bounding the expected information gain.}
From the last step, we are left with the sum of expected information gains, $\sum_{t=1}^T \EE[I_t|\fF_{t-1}]$. Remember that $\sum_{t=1}^T I_t \leq \gamma_T$. Using Lemma \ref{lemma: concentration of conditinal means}, we show that the sum of expected information gains $\sum_{t=1}^T \EE[I_t|\fF_{t-1}]$ is similarly bounded.

\begin{lemma}\label{lemma: concentration of expected information gain}
	With probability at least $1-\delta$,
	\begin{align*}
	\forall\; T \in \NN, \quad \sum_{t=1}^T \EE[I_t|\fF_{t-1}] &\leq 2\gamma_T + 4 \gamma_T \log \frac{1}{\delta} + 8 \gamma_T \log(4\gamma_T) + 1 \text{ ,}
	\end{align*}
	and if further $I_t \leq 1$ holds, then, with probability at least $1-\delta$,
	\begin{align*}
	\forall\; T \in \NN, \quad\sum_{t=1}^T \EE[I_t|\fF_{t-1}] &\leq 2\gamma_T + 4 \log \frac{1}{\delta} + 8 \log(4) + 1\text{ .}
	\end{align*}
\end{lemma}

\begin{proof}\textbf{of Lemma \ref{lemma: concentration of expected information gain}.} Remember that $\sum_{t=1}^T I_t \leq \gamma_T$ by definition. The proof of the lemma is a direct application of Lemma \ref{lemma: concentration of conditinal means}, first with the conservative bound $I_t \leq \sum_{s=1}^t I_s \leq \gamma_t$, and then with the assumption $I_t \leq 1$.
\end{proof}

\paragraph{Step 4) Completing the proof.}
	For the proof of the theorem, we take the union bound over the events such that Lemma \ref{lemma: bound regret by instantaneous regret} and Lemma \ref{lemma: concentration of expected information gain} simultaneously hold, effectively replacing $\delta$ by $\frac{\delta}{2}$ in respective statements. Applying Lemma \ref{lemma: bound regret by instantaneous regret}, \ref{lemma: expected regret by information ratio} and \ref{lemma: concentration of expected information gain} in order, we find
	\begin{align*}
	R_T &\leq \tfrac{5}{4}  \sum_{t=1}^T \EE[\Delta_t|\fF_{t-1}] + 4 S \log\left(\tfrac{8\pi^2 T^2}{3\delta} (\log(T) + 1)\right)\\
	& \leq \tfrac{5}{4}  \sqrt{\sum_{t=1}^T \Psi_t \; \sum_{t=1}^T \EE[I_t|\fF_{t-1}]} + 4 S \log\left(\tfrac{8\pi^2 T^2}{3\delta} (\log(T) + 1)\right)\\
	& \leq \tfrac{5}{4}  \sqrt{\sum_{t=1}^T \Psi_t \; \left(2\gamma_T + 4 \gamma_T \log \frac{2}{\delta} + 8 \gamma_T \log(4\gamma_T) + 1\right)} + 4 S \log\left(\tfrac{8\pi^2 T^2}{3\delta} (\log(T) + 1)\right) \text{ ,}
	\intertext{and if $I_t \leq 1$ holds, the last step can be strengthend to}
	R_T &\leq\tfrac{5}{4}  \sqrt{\sum_{t=1}^T \Psi_t \; \left(2\gamma_T + 4 \log \frac{2}{\delta} + 8 \log(4) + 1\right)} + 4 S \log\left(\tfrac{8\pi^2 T^2}{3\delta} (\log(T) + 1)\right) \text{ .}
	\end{align*}
	This completes the proof. \jmlrQED
	\noindent Dropping all constants, we find
	\begin{align*}
	R_T & \leq \oO\left( \sqrt{\sum_{t=1}^T \Psi_t \;\left(\gamma_T \log \frac{1}{\delta} +  \gamma_T \log(\gamma_T) \right)} +  S \log\left(\tfrac{T \log(T)}{\delta} \right)\right) \text{ ,}
	\intertext{and if $I_t \leq 1$ holds}
	R_T &\leq \oO\left(\sqrt{\sum_{t=1}^T \Psi_t \; \left(\gamma_T +  \log \frac{1}{\delta}\right)} + S \log\left(\tfrac{T \log(T)}{\delta}\right) \right) \text{ .}
	\end{align*}
	
	We remark that in Step 3) in our proof, a dependence on $\log \frac{1}{\delta}$ is unavoidable, as shown by the following example.
\begin{example}\label{exa: counter example conditional means almost surely}
	Define a stochastic process $I_t$, such that $I_t \sim \text{Ber}(\frac{1}{2})$, if $I_1 = 0, \dots, I_{t-1} = 0$, and $I_t= 0$ else. Then $\sum_{t=1}^T \EE_{\mu}[I_t|\fF_{t-1}] = \Omega(T)$ with positive probability, but $\sum_{t=1}^T I_t \leq 1$ always holds.
\end{example}

\section{Properties of the Information Regret Ratio}\label{app: properties of the regret-information ratio}
Here we present details on the properties of the regret-information ratio and its surrogate. In the next lemmas, we show that a minimizer of $\Psi_t^+(\mu)$ always exists, and in fact can be chosen to be supported on two actions, and further, that the regret-information ratio is a convex function of $\mu \in \pP(\xX)$. Formally, we assume that $\xX$ is a compact, metric space and $\Delta:\xX\rightarrow \RRnn$, $g:\xX \rightarrow \RRnn$ are continuous, bounded functions, such that $g$ is not zero everywhere. In the following, we let $\Psi : \pP(\xX) \rightarrow \RR$ be any function of the form $\Psi(\mu) = \frac{\EE_\mu[\Delta(x)]^2}{\EE_{\mu}[g(x)]}$, which includes both the regret-information ratio and its surrogate. Note however, that minimizing the actual regret-information ratio is a trivial matter (assuming that $\Delta$ is known), as for $x^* \in \argmax_{x \in \xX}f(x)$, we have by definition $\Delta(x^*)=0$, and therefore $\min_{x \in \xX} \Psi_t(x) = 0$; but this is of course not true in general for the surrogate $\Psi_t^+$.

\begin{lemma}\label{lemma: minimizer of psi exists}
 	There exists a $\mu^* \in \pP(\xX)$ such that $\Psi(\mu^*) = \inf_{\mu \in \pP(\xX)} \Psi(\mu)$.
\end{lemma}
\begin{proof} Clearly $\Psi(\mu) \geq 0$, and if there is a point $x \in \xX$ with $\Delta(x) = 0$, we have $\Psi(\mu^*) = 0$ for $\mu^* = \delta_x$, with the convention that $\tfrac{0}{0} = 0$. Therefore, without loss of generality, we can assume that $\Delta(x) \geq \epsilon > 0$ for all $x \in \xX$.
	
	Define $h : \RR\times\RRp \rightarrow \RR, (x,y) \mapsto x^2/y$, such that $\Psi(\mu) = h(\EE_{\mu}[\Delta], \EE_{\mu}[g])$. Denote by $\Psi^* = \inf_{\mu \in \pP(\xX)} \Psi(\mu)$ and pick a sequence $\mu_n$ such that $\Psi(\mu_n) \rightarrow \Psi^*$ as $n \rightarrow \infty$. By the assumption that $\Delta \geq \epsilon$ and $g$ is not zero everywhere, $\Psi^* < \infty$, and in particular $\EE_{\mu_n}[g] > 0$ for $n$ large enough. Since $\mu_n$ is a probability distribution over a compact, metric space, the sequence is trivially \emph{tight}. We apply Prokhorov's theorem \cite[]{ProkhorovConvergenceRandomProcesses1956}, which guarantees the existence of a subsequence $\mu_{n_i}$ converging weakly to some $\mu^* \in \pP(\xX)$ (in fact, $\pP(\xX)$ is compact in the weak* topology). By definition of weak convergence (or weak* convergence from the point of view of functional analysis), $\int s(x) \,d\mu_{n_i} \rightarrow \int s(x) d \mu$ for every continuous bounded function $s: \xX \rightarrow \RR$. Hence $\EE_{\mu_{n_i}}[\Delta] \rightarrow \EE_{\mu^*}[\Delta]$, and the same for  $\EE_{\mu_{n_i}}[g] \rightarrow \EE_{\mu^*}[g]$. Finally, continuity of $h$ gives 
	\begin{align*}
	\Psi^* = \lim_{i \rightarrow \infty} h(\EE_{\mu_{n_i}}[\Delta], \EE_{\mu_{n_i}}[g]) = h(\EE_{\mu^*}[\Delta], \EE_{\mu^*}[g]) = \Psi(\mu^*) \text{ ,}
	\end{align*} 
	which proves the claim.
\end{proof}

In the next two lemmas, we extend Proposition 6 of \cite{RussoLearningOptimizeInformationDirected2014} to the case of compact action sets $\xX$.
\begin{lemma}\label{lemma: psi is convex}
	$\Psi(\mu)$ is convex in $\mu$.
\end{lemma}
\begin{proof}
	The proof is completely analogs to the proof of Proposition 6 of \cite{RussoLearningOptimizeInformationDirected2014}. Define $h : \RR\times \RRp \rightarrow \RR, (x,y) \mapsto x^2/y$. As shown in Chapter 3 of \cite{BoydConvexoptimization2004},  $h$ is convex on its domain. Further, $f:\pP(X) \rightarrow \RR^2, \mu \mapsto (\EE_{\mu}[\Delta], \EE_{\mu}[g])$ is an affine function. Since $\Psi(\mu) = h(f(\mu))$ is a composition of a convex and an affine function, $\Psi$ is convex, too.
\end{proof}
\begin{lemma*}\textnormal{\textbf{\ref{lemma: minimizer supported on at most two atoms}} \textnormal{(in Section \ref{section: information directed sampling}})} 
	There exists a measure $\mu^* \in \argmin_{\mu \in \pP(\xX)} \Psi(\mu)$ with $|\supp(\mu^*)| \leq 2$, i.e.\ a measure which achieves the minimum value of $\Psi$ and is supported on at most two actions.
\end{lemma*} 
\begin{proof} The proof is an extension of Proposition 6 in \citep{RussoLearningOptimizeInformationDirected2014}, where we deal with the technicalities of having a compact action set.
	
	Define $\Psi^* = \min_{\mu \in \pP(\xX)} \Psi(\mu)$. Again, without loss of generality, we can assume that $\Delta(x) \geq \epsilon > 0$ for all $x\in \xX$, as otherwise we immediately get a measure supported on a single action, such that $\Psi^* = 0$. In particular, for what follows, we can assume that $\Psi^* > 0$. Define $\eta(\mu) = \EE_\mu[\Delta]^2  - \EE_\mu[g]\Psi^*$. As a first step, we show that $\eta$ has the same minimizers as $\Psi$. 
	Note that $\eta(\mu) \geq 0$, and if $\Psi(\mu) = \Psi^*$, then $\eta(\mu) = 0$, which shows one direction. For the converse, assume that $\mu$ minimizes $\eta$, i.e.\ $\eta(\mu) = 0$. This immediately gives $\Psi(\mu) = \Psi^*$. 
	
	Now, let $\mu$ be any minimizer of $\Psi$, and therefore also a minimizer of $\eta$. Remember that by definition of the support, for any $x \in \supp(\mu)$ and any neighborhood $N_x$ of $x$, $\mu(N_x) > 0$. We claim that the function 
	\begin{align*}
	s : \xX \rightarrow \RR,\quad s(x) = 2\; \EE_\mu[\Delta] \Delta(x)- \Psi^* g(x)
	\end{align*}
	is constant on $\supp(\mu)$. Suppose there exist $x,y \in \supp(\mu)$ such that $s(x) < c < s(y)$ for some $c \in \RR$. Then, by continuity of $s$, we can find disjoint neighborhoods $N_x$ and $N_y$ of $x$ and $y$ respectively, such that for all $x' \in N_x$ and $y' \in N_y$, $s(x') < c < s(y')$. 

	Let $\mu_x$, $\mu_y$ and $\mu_\bot$ be the measures obtained by restricting $\mu$ to $N_x$, $N_y$ and $(N_x \cup N_y)^c$ respectively, and let $\alpha_x = \mu_x(N_x)$ and $\alpha_y = \mu_y(N_y)$. By the definition of the support, $\alpha_x > 0$ and $\alpha_y > 0$ (in particular, $\mu_x/\alpha_x$ and $\mu_y/\alpha_y$ are the conditional probability distributions of $\mu$ given $N_x$ and $N_y$ respectively). Note that for our choice of $x$ and $y$,
	\begin{align}
	\frac{1}{\alpha_x}	\int_{N_x} s(x)\; d\mu_x < c  < \frac{1}{\alpha_y} \int_{N_y} s(y)\; d\mu_y \text{ .}\label{eq: E[s] on N_x and N_y different}
	\end{align}
	Define for $\lambda \in [-\alpha_x^{-1}, \alpha_y^{-1}]$ a parametrized measure
	\begin{align*}
	\mu_\lambda = \mu_\bot + (1 - \lambda \alpha_y) \mu_x + (1+ \lambda \alpha_x)\mu_y \text{ .}
	\end{align*}
	Clearly $\mu_\lambda(\xX) = 1$, hence $\mu_\lambda$ defines a probability measure for $-\alpha_x^{-1} \leq \lambda \leq \alpha_y^{-1}$ such that $\mu_0 = \mu$. Intuitively, by varying $\lambda$ we transfer probability mass from $N_x$ to $N_y$. With a slight abuse of notation, we write $\eta(\lambda) = \eta(\mu_\lambda)$. Because $\mu$ is a minimizer of $\eta$, it must be that $\frac{d}{d\lambda}\eta(\lambda)|_{\lambda=0} = 0$. We have that
	\begin{align*}
	\eta(\lambda) &= \EE_{\mu_\lambda}[\Delta]^2 - \Psi^*\;\EE_{\mu_\lambda}[g] \\
	&= \EE_{\mu_\bot + (1- \lambda\alpha_y) \mu_x + (1 + \lambda\alpha_x)\mu_y}[\Delta]^2  - \Psi^* \; \EE_{\mu_\bot + (1- t\alpha_y) \mu_x +  (1 + \lambda\alpha_x)\mu_y} [g] \text{ ,}
	\end{align*}
	and we calculate the derivative with respect to $\lambda$,
	\begin{align*}
	\frac{d}{d\lambda}\eta(\lambda) &= 2 \EE_{\mu_\lambda}[\Delta] \int_\xX \Delta \; (\alpha_x d\mu_y - \alpha_yd\mu_x) + \Psi^* \int_\xX g\; (\alpha_x d\mu_y - \alpha_y \mu_x) \text{ .}
	\end{align*}
	Evaluating the previous expression at $\lambda=0$ yields
	\begin{align*}
	0 = 2 \EE_{\mu}[\Delta] \int_\xX \Delta \; (\alpha_x d\mu_y - \alpha_y  d\mu_x ) + \Psi^* \int_\xX g\; (\alpha_x d\mu_y - \alpha_y d\mu_x) \text{,}
	\end{align*}
	which contradicts \eqref{eq: E[s] on N_x and N_y different}. This shows that $d^* := s(x)$ is constant for $x \in \supp(\mu)$, and in particular for all $x \in \supp(\mu)$,
	\begin{align*}
	g(x) = - \frac{d^*}{\Psi^*} + \frac{2\EE_\mu[\Delta]}{\Psi^*} \Delta(x) \text{ .}
	\end{align*}
	To complete the proof, note that $\supp(\mu)$ is compact because the support is always closed and $\xX$ compact; hence we can pick $x_{\min} \in \argmin_{x \in \supp(\mu)} g(x)$ and $x_{\max} \in \argmax_{x \in \supp(\mu)} g(x)$. Let $0 \leq \beta \leq 1$, such that
	\begin{align*}
	\EE_{\mu}[g] = \beta g(x_{\min}) + (1-\beta) g(x_{\max}) \text{ .}
	\end{align*}
	Finally, $\mu^* = \beta \delta_{x_{\min}} + (1-\beta) \delta_{x_{\max}}$ is a distribution supported on $\{x_{\min}, x_{\max}\}$. From the last two equations, it follows that $\EE_{\mu}[g] = \EE_{\mu^*}[g]$, and also  $\EE_{\mu}[\Delta] = \EE_{\mu^*}[\Delta]$. Consequently, $\mu^*$ achieves the same objective value on $\eta$ as $\mu$, and therefore is a minimizer of both $\eta$ and $\Psi^*$.
\end{proof}

\section{Least Squares Estimation with Heteroscedastic Noise} \label{app: least squares estimation with heteroscedastic noise}
\subsection{Least Squares Estimator for RKHS}

Note that for any linear operator $A : \hH \rightarrow \RR^T$,
\begin{align*}
(A^*A + \lambda I_\hH)^{-1}A^* = A^* (AA^* + \lambda I_{\RR^T})^{-1} \text{ .}
\end{align*}
Remember that we defined $(K_T)_{i,j} := (MM^*)_{i,j} = k(x_i,x_j)$ and $k_T(x)_i := (Mk_x)_i = \<x_i, x\> = k(x_i,x)$. Using the reproducing property and the equation above with $A = \Sigma^{-1}M$, we calculate $\hat{\mu}_T(x)$ for any $x \in \RR^d$,
\begin{align*}
\hat{\mu}_T(x) &= \<\hat{\mu}_T, k_x \>_\hH\\
&= \<(M^*\Sigma^{-1}M +  \lambda I_\hH)^{-1} M^*\Sigma^{-1} y, k_x \>_\hH\\
&= \<M^*\Sigma^{-1/2}(\Sigma^{-1/2}MM^*\Sigma^{-1/2} + \lambda I_{\RR^n})^{-1} \Sigma^{-1/2} y, k_x \>_\hH\\
&= \<M^*(MM^* + \lambda \Sigma)^{-1} y, k_x \>_{\hH}\\
&=  \<(MM^* + \lambda \Sigma)^{-1} y, M k_x \>_{\RR^T}\\
&= k_T(x)^\T (K_T + \lambda \Sigma)^{-1}y \text{ .}
\end{align*}
Further, observe that for $A = \Sigma^{-1/2}M$,
\begin{align*}
\lambda(A^*A + \lambda I)^{-1} = I - A^*(A A^* + \lambda I)^{-1}A \text{ .}
\end{align*}
In particular, we have $\lambda \<k_x,k_y\>_{V_T^{-1}} = k(x,y) - k_T(x)^\T (K_T + \lambda \Sigma_T)^{-1}k_T(y)$, hence using $v = k_x = k_y$, we can calculate $\|v\|_{V_T^{-1}}^2$ as follows,
\begin{align*}
\|v\|_{V_T^{-1}}^2 = \frac{1}{\lambda} \left(k(x,x) - k_T(x)^\T(K + \lambda \Sigma)^{-1}k_T(x)\right) = \sigma_T(x)^2 \text{ .}
\end{align*}
\section{Summary of Information Gain Functions}\label{app: information functions}
\paragraph{IDS-F} In the main text, we have defined $I_t^\F(x) = \log\left(1 + \frac{\sigma_t(x)^2}{\rho(x)^2}\right)$, which in the Bayesian setting, with Gaussian prior and likelihood and posterior variance $\sigma_t(x)$, corresponds to the mutual information $\II(x;f|\fF_{t-1})$.
\paragraph{IDS-UCB}Remember that we defined $\sigma_t(y|x)^2$ as the confidence width at $y \in \xX$ if we were to evaluate $x\in \xX$ at time $t$. Then IDS-UCB is defined through
\begin{align*}
I_t^\UCB(x) = \log\left(\frac{\sigma_t(x_t^\UCB)^2}{\sigma_t(x_t^\UCB|x)^2}\right) \text{ .}
\end{align*} 
Using the Sherman-Morrison formula, we compute for the linear case, denoting $\sigma_t(x) = \|x\|_{V_t^{-1}}$,
\begin{align*}
\sigma_t(x_t^\UCB|x)^2 = \sigma(x_t^\UCB)^2 - \frac{(x^\T V_t^{-1}x_t^\UCB)^2}{\rho(x)^2 + \sigma_t(x)^2} \text{ .,}
\end{align*}
Using Schur's complement, for RKHS functions we get 
\begin{align*}
\sigma_t(x_t^\UCB|x)^2 = \sigma(x_t^\UCB)^2 - \frac{k_t(x,x_t^\UCB)^2}{\rho(x)^2 + \sigma_t(x)^2} \text{ ,}
\end{align*}
compare Appendix F in \citep{ChaudhuriActiveHeteroscedasticRegression2017}. From the last expression, it is also easy to see that
\begin{align*}
I_t^\UCB(x_t^\UCB)= I_t^\F(x_t^\UCB) =  \log\left(1 + \frac{\sigma(x_t^\UCB)^2}{\rho(x_t^\UCB)^2}\right)\text{ .} 
\end{align*}
\paragraph{IDS-TS}
Similar as IDS-UCB, we define IDS-TS by using the information function
\begin{align*}
I_t^\TS(x) = \log\left(\frac{\sigma_t(x_t^\TS)^2}{\sigma_t(x_t^\TS|x)^2}\right) \text{ ,}
\end{align*} 
where $x_t^\TS$ is a \emph{proposal action} by the Thompson Sampling policy, that is, $x_t^{\TS}$ maximizes a sample $\tilde{f}$ from the posterior distribution of $f$. Intuitively, we use Thompson Sampling to identify an action to gather information on, but the resulting IDS-TS is not restricted to playing $x_t^\TS$. Though currently we do not have a bound for the regret-information ratio of this policy, we point out that if the regret-information ratio of Thompson Sampling is bounded, then by a similar argument as in Lemma \ref{lemma: bounding the regret information surrogate}, one can also bound the regret-information ratio of IDS-TS.

\paragraph{IDS-E}
We also propose the use of the averaged information gain
\begin{align*}
I_t^E(x) = \frac{1}{m} \sum_{i=1}^m \log\left(\frac{\sigma_t(x_{t,i}^\TS)^2}{\sigma_t(x_{t,i}^\TS|x)^2}\right) \text{ ,}
\end{align*}
calculated using $m$ proposals $x_{t,1}^\TS, \dots, x_{t,m}^\TS$ from the Thompson Sampling policy.

\section{Further Numerical Experiments}\label{app: further experiments}

\begin{figure}
	\centering   	
	\hspace{-15pt}
	\subfigure[Simulation of Example \ref{exa: ucb and ts fail}][t]{\label{fig: homoscedastic_C}\includegraphics[scale=1,trim={0 0 0 0},clip]{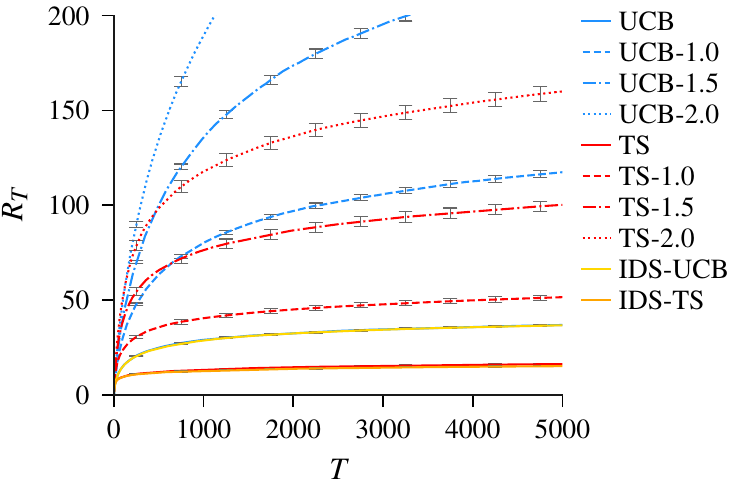}}
	\hspace{10pt}
	\subfigure[Heteroscedastic Noise][t]{\label{fig: heteroscedastic_A}\includegraphics[scale=1,trim={0.5cm 0 0.0cm 0},clip]{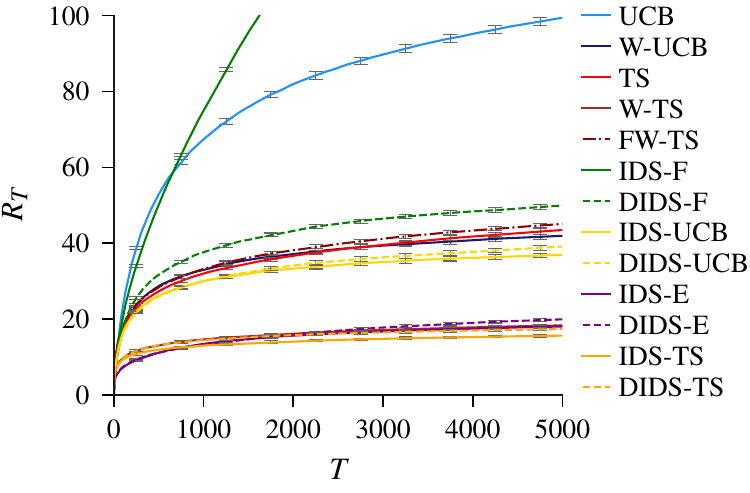}}
	\hspace{-20pt}
	\caption{Left: We show a simulation of Example \ref{exa: ucb and ts fail} with pairs of correlated actions. For each pair, one of the action yields low observation noise $\nN(0, \sigma_0^2)$ with $\sigma_0=0.5$, whereas the second action has increased observation noise $\nN(0, \sigma^2)$ for $\sigma \in \{1, 1.5, 2\}$. Without the second action UCB and IDS-UCB perform similarly, and the same holds true for TS/IDS-TS. Assuming that UCB/TS always commit to the noisier action among each pair, UCB-$\sigma$ and TS-$\sigma$  show the regret achieved given increased noise variance $\sigma^2$. As expected, the performance quickly deteriorates, because the confidence intervals concentrate at a slower rate. On the other hand, IDS-UCB/IDS-TS are guaranteed to always pick the less noisy action among each pair.\\
	Right: We show another instance of our heteroscedastic experiment, including also further variants of IDS. We observed that the randomized versions of IDS-TS, IDS-E and IDS-UCB perform very similar to the respective deterministic DIDS-TS, DIDS-E and DIDS-UCB. Only for IDS-F and DIDS-F the performance is drastically different, where as noted before IDS-F is over-exploring. Again, one can see that W-UCB outperforms UCB, and W-TS outperforms TS, which shows the benefit of using the weighted least squares estimator in the heteroscedastic setting. In this case, W-TS is actually competitive with IDS-TS (which can hardly be seen due to overlapping lines). Finally, IDS-UCB clearly outperforms W-UCB, and IDS-TS achieves the lowest regret among all policies after 5000 steps. }
\end{figure}

We simulated the setting of Example \ref{exa: ucb and ts fail}, and also evaluated the IDS variants proposed in Appendix \ref{app: information functions} in the same basic setup as described in Section \ref{section: experiments}, but with a different sample of randomly generated actions (Figure \ref{fig: homoscedastic_C}, \ref{fig: heteroscedastic_A}). Further, in the heteroscedastic case, we instead choose the noise bound for each action uniformly in $[0.4, 0.6]$, a setting closer to the homoscedastic case. We also included an additional variant of Thomson Sampling (FW-TS), which engages in the oversampling as needed for the frequentist guarantees of \cite{AgrawalThompsonSamplingContextual2013} and \cite{AbeilleLinearThompsonSampling2017}, in addition to using the weighted least squares estimator. IDS-E is averaging over 10 proposals of Thompson Sampling.

We stress that for the strategies TS, WTS, D/IDS-TS and D/IDS-E there are currently no frequentist regret bounds known, still they perform well in the experiments. The difference between the IDS and DIDS versions is small except for IDS-F and DIDS-F, where as previously pointed out, the randomized version seems to engage in over-exploration along the theoretical upper bound. Similar, as IDS-F is much more efficient than DIDS-F in the sense that it achieves a higher information gain in terms if $I_t^F$ while still maintaining the worst-case regret guarantee, we conjecture that there are settings and information functions, where the randomized IDS has lower regret than the deterministic version. As usual, all experiments have to be taken with a grain of salt; of course it is possible to degrade the performance of information directed strategies by setting the noise of a suboptimal action to a very small value; and the other way round, if the optimal action has very little noise, our methods gain an additional advantage. To compile a more thorough experimental study, also including real-world applications, is a task for future work.

\section{Useful Inequalities}\label{app: useful inequalities}
The following lemma gives an upper bound on the variance of a bounded random variable.
\begin{lemma}[\cite{Bhatiabetterboundvariance2000}]\label{lemma: bhatia-davis}
	Let $X$ be a real random variable supported in $[m,M]$. Then, 
	\begin{align*}
	\Var(X) \leq (M-\EE[X])(\EE[X] - m) \text{ ,}
	\end{align*}
	and the bound is tight, if all mass is concentrated on the end-points of the interval.
\end{lemma}

\noindent
The next two lemmas are two technical inequalities.
\begin{lemma}\label{lemma: 2 sqrt ab < a + b}
	Let $a,b \geq 0$. Then $2\sqrt{ab} \leq a + b$.
\end{lemma}
\begin{proof}
	Squaring both sides gives $4 ab \leq a^2 + b^2 + 2 ab$, or equivalently, $0 \leq a^2 + b^2 - 2ab = (a-b)^2$, which is clearly true.
\end{proof}

\begin{lemma}\label{lemma: sqrt ab ln(ab) < a + b ln(b)}
	Let $a,b \geq 1$. Then $\sqrt{ab \log(ab)} \leq a + b \log(b)$.
\end{lemma}
\begin{proof}
By $a,b \geq 1$, we can equivalently show that $ab \log(ab) \leq (a + b\log(b))^2 =  a^2 + 2ab\log(b) + (b \log(b))^2$. Assume first that $b \log(a) \leq a$. Then 
\begin{align*}
	ab \log(ab) = ab \log(a) + ab \log(b) \leq a^2 + ab \log(b) \leq (a + b\log(b))^2 \text{ .}
\end{align*}
Assume now the opposite, i.e.\ $a \leq b \log(a)$. With the inequality $\log(a) \leq \sqrt{a}$, this implies that $a \leq b^2$. Consequently,
\begin{align*}
	ab \log(ab) \leq ab \log(b^3) = 3 ab \log(b) < 4 ab \log(b)
\end{align*}
Hence, by Lemma \ref{lemma: 2 sqrt ab < a + b},
\begin{align*}
	\sqrt{ab \log(ab)} \leq 2 \sqrt{ ab \log(b)} \leq a + b\log(b) \text{ .}
\end{align*}
This completes the proof.
\end{proof}

\end{document}